\documentclass{article}

\usepackage{microtype}
\usepackage{graphicx}
\usepackage{subfigure}
\usepackage{booktabs} % for professional tables

\usepackage{hyperref}

\usepackage{amsmath,amsfonts,amssymb,amsthm,commath,dsfont}
\usepackage{comment}
\usepackage[page]{appendix}
\usepackage{array,multirow}
\usepackage{color}
\usepackage{enumitem}
\usepackage{mathtools}
\usepackage{wrapfig}

\usepackage{graphicx}
\usepackage{caption}
\usepackage{epstopdf,xcolor}
\usepackage{bbm}
\usepackage[colorinlistoftodos,prependcaption]{todonotes}
\usepackage{dsfont}
\usepackage{multicol}
\usepackage{amsmath,amsfonts,commath}
\usepackage{array,multirow}
\usepackage{mathrsfs} 

\usepackage{algorithm}
\usepackage{algpseudocode}

\newcommand{\Pa}{\textit{Pa}}

\newcommand{\supp}{\text{supp}}

\newtheorem{theorem}{Theorem}
\newtheorem{lemma}{Lemma}

\newtheorem{proposition}{Proposition}
\newtheorem{corollary}{Corollary}
\newtheorem{definition}{Definition}
\newtheorem{example}{Example}
\newtheorem{remark}{Remark}
\newtheorem{claim}{Claim}
\newtheorem{assumption}{Assumption}

\usepackage[accepted]{icml2020}

\icmltitlerunning{Characterizing Distribution Equivalence and Structure Learning for Cyclic and Acyclic Directed Graphs}

\begin{document}

\twocolumn[
\icmltitle{Characterizing Distribution Equivalence and Structure Learning for\\ Cyclic and Acyclic Directed Graphs}

\icmlsetsymbol{equal}{*}

\begin{icmlauthorlist}
\icmlauthor{AmirEmad Ghassami}{to}
\icmlauthor{Alan Yang}{to}
\icmlauthor{Negar Kiyavash}{ed}
\icmlauthor{Kun Zhang}{goo}

\end{icmlauthorlist}

\icmlaffiliation{to}{Department of Electrical and Computer Engineering, University of Illinois at Urbana-Champaign, Urbana, IL, USA}
\icmlaffiliation{ed}{College of Management of Technology, \'Ecole Polytechnique F\'ed\'erale de Lausanne (EPFL), Switzerland}
\icmlaffiliation{goo}{Department of Philosophy, Carnegie Mellon University, Pittsburgh, PA, USA}

\icmlcorrespondingauthor{AmirEmad Ghassami}{ghassam2@illinois.edu}

\icmlkeywords{Machine Learning, ICML}

\vskip 0.3in
]

\printAffiliationsAndNotice{}

\begin{abstract}
The main approach to defining equivalence among acyclic directed causal graphical models is based on the conditional independence relationships in the distributions that the causal models can generate, in terms of the Markov equivalence. However, it is known that when cycles are allowed in the causal structure, conditional independence may not be a suitable notion for equivalence of two structures, as it does not reflect all the information in the distribution that is useful for identification of the underlying structure. In this paper, we present a general, unified notion of equivalence for linear Gaussian causal directed graphical models, whether they are cyclic or acyclic. In our proposed definition of equivalence, two structures are equivalent if they can generate the same set of data distributions. We also propose a weaker notion of equivalence called quasi-equivalence, which we show is the extent of identifiability from observational data. We propose analytic as well as graphical methods for characterizing the equivalence of two structures. Additionally, we propose a score-based method for learning the structure from observational data, which successfully deals with both acyclic and cyclic structures.
\end{abstract}

%\vspace{-7mm}
\section{Introduction}
\label{sec:intro}
%\vspace{-2mm}

The problem of learning directed graphical models from data has received a significant amount of attention over the past three decades since those models provide a compact and flexible way to represent constraints on the joint distribution of the data \cite{koller2009probabilistic}. When interpreted causally, they can model causal relationships among the variables of the system and help make predictions under intervention \cite{pearl2009causality,spirtes2000causation}.

There exists an extensive literature on learning causal graphical models from observational data under the assumption that the model is a directed acyclic graph (DAG). 
Existing approaches include constraint-based methods \cite{spirtes2000causation,pearl2009causality}, score-based methods \cite{heckerman1995learning,chickering2002optimal},  hybrid methods \cite{tsamardinos2006max}, as well as methods which make extra assumptions on the data generating process. For example, the model may be assumed to be linear with non-Gaussian exogenous noise variables \cite{shimizu2006linear} or contain specific types of non-linearity in the causal modules \cite{hoyer2009nonlinear,zhang2008distinguishing}. 

\begin{wrapfigure}{r}{0.15\textwidth}
  \vspace{-6mm}
  \begin{center}
    \includegraphics[scale=0.43]{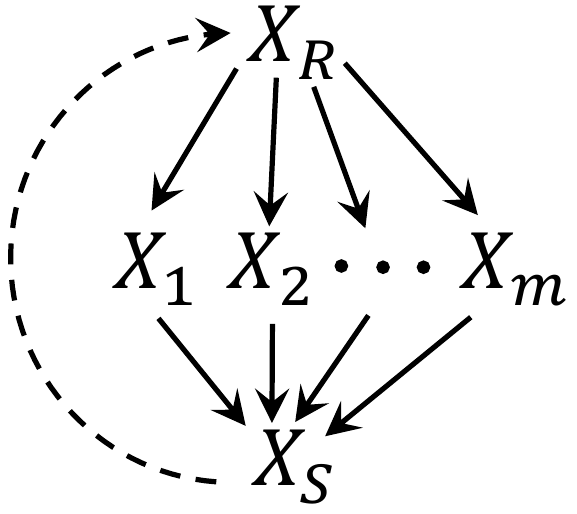}
  \end{center}
    \vspace{-5mm}
  \caption{}
  \vspace{-2mm}
  \label{fig:ignore}
\end{wrapfigure}
Most real-life causal systems contain feedback loops, since feedback is generally required to stabilize the system and improve performance in the presence of noise. Hence, the causal directed graph (DG) corresponding to such systems will be cyclic \cite{spirtes1995directed,hyttinen2012learning}. However, there are relatively few works on learning structures that contain cycles. In many state-of-the-art causal models, not only is feedback ignored, it is also explicitly assumed that there are no cycles passing information among the considered quantities. 
Note that ignoring cycles in structure learning can be very consequential. For instance, in Figure \ref{fig:ignore}, if one uses a conditional independence-based learning method designed for DAGs such as the PC algorithm \cite{spirtes2000causation}, in the absence of the dashed feedback loop the skeleton will be estimated correctly on the population dataset and the directions for all edges into $X_S$ can be determined.
However, in the presence of the feedback loop, the output is a complete directed graph since no two variables will be independent conditioned on any subset of the rest of the variables. 

The lack of attention to cyclic structures in the literature is primarily due to the simplicity of working with acyclic models (see \cite{spirtes1995directed}) and the fact that in contrast to DAGs, there exists no generally accepted characterization of statistical equivalence among cyclic structures in the literature.
The main method for defining equivalence among DAGs is based on the conditional independence (CI) relationships in the distributions that they imply. That is, two DAGs are equivalent if and only if they imply the same CI relations. CI relationships can be seen from statistical data, and the CI-based equivalence characterization for DAGs is attractive because CI relationships contain all the information in the distribution that can be used for structure learning under the assumption of causal sufficiency.
However, when causal sufficiency is violated or cycles are allowed in the structure, conditional independency may not reflect all the information in the distribution that can be used to identify the underlying structure.
That is, the joint distribution may contain information that can be used to distinguish among the members of a CI-based equivalence class, which is also known as a Markov equivalence class. This means that it is possible for two graphs to be distinguishable from observational data even though they are in the same Markov equivalence class.
For more details, see \cite{lacerda2008discovering} for the case of the violation of acyclicity and \cite{tian2002testable,shpitser2014introduction} for the case of the violation of causal sufficiency.

With the goal of bridging the gap between cyclic and acyclic DGs, in this paper we present a general characterization of equivalence for linear Gaussian DGs.\footnote{Note that for non-linear cyclic SEMs, even the Markov property does not necessarily hold \cite{spirtes1995directed,pearl1996identifying,neal2000deducing}, and hence, it is not clear if one can make general statements about the equivalence of structures regardless of the involved equations.} 
In the case of DAGs, our approach provides a novel alternative to the customary tests for Markov equivalence.
The proposed distribution equivalence characterization (Theorems \ref{thm:mainrot2} and \ref{thm:graphical}) not only is capable of characterizing equivalence beyond conditional independencies, but also provides a simpler and more concise evaluation approach compared to \cite{richardson1996polynomial}. 
We summarize our contributions as follows.
%\vspace{-4mm}
\begin{itemize}%[leftmargin=5mm,noitemsep]
\item We present a general, unified notion of equivalence based on the set of distributions that the directed graphs are able to generate (Section \ref{sec:desc}). In our proposed definition of equivalence, two structures are equivalent if they can generate the same {\it set} of data distributions. 
\item We propose an algebraic and graphical characterization of the equivalence of two DGs, be they cyclic or acyclic, based on the so-called Givens rotations (Sections \ref{sec:test} and \ref{sec:testgraphical}). 
\item We also propose a weaker notion of equivalence called quasi-equivalence, which we show is the extent of identifiability from observational data (Section \ref{sec:learning}). 
\item We propose a score-based method for structure learning from observational data with local search. We show that our score asymptotically achieves the extent of identifiability (Section \ref{sec:learning}). To the best of our knowledge, this is the first local search method capable of learning structures with cycles. The implementation is publicly available at \href{https://github.com/syanga/dglearn}{https://github.com/syanga/dglearn}.
\end{itemize}
%\vspace{-3.5mm}

\subsection{Related Work}

%{\bf Related Work.}
\citet{richardson1996discovery,richardson1996polynomial} proposed graphical constraints necessary and sufficient for Markov equivalence for general cyclic DGs and proposed a constraint-based algorithm for learning cyclic DGs. That algorithm was later extended to handle latent confounders and selection bias \cite{strobl2019constraint}.
\citet{hyttinen2013discovering,hyttinen2014constraint}
also focused on structure learning based on CI relationships for possibly cyclic and causally insufficient data gathered from multiple domains that may contain conflicting CI information. They proposed an approach based on an SAT or ASP solver. Due to generality of their setup, the run time of this approach can be restricting.
A similar approach was proposed in \cite{forre2018constraint} for the case of nonlinear functional relationships with an extended notion of graphical separation called $\sigma$-separation.
Also, \citet{hyttinen2012learning} provided an algorithm for learning linear models with cycles and confounders
that deals with perfect interventions.  
As mentioned earlier, having the assumption of non-Gaussian exogenous noises and specific types of non-linearity may lead to unique identifiability in DAGs. This idea was also investigated for cyclic DGs. \citet{lacerda2008discovering} proposed a method for learning DGs based on the ICA approach for linear systems with non-Gaussian exogenous noises, and \citet{mooij2011causal} investigated the case of nonlinear causal mechanisms with additive noise.

To the best of our knowledge, there exists no work on learning cyclic linear Gaussian models which utilizes the observational joint distribution itself rather than CI relationships in the distribution.

%\vspace{-3mm}
\section{Distribution Equivalence}
\label{sec:desc}
%\vspace{-2mm}

We consider a linear structural causal model over $p$ observable variables $\{X_i\}_{i = 1}^p$, with exogenous Gaussian noise. For $i\in[p]$, variable $X_i$ is generated as $X_i= \sum_{j=1}^pB_{j,i} X_j+N_i$, in which $N_i$ is the exogenous noise corresponding to variable $X_i$. We assume that $B_{i,i}=0$, for all $i\in[p]$. 
Variable $X_j$ is a direct cause of $X_i$ if $B_{j,i}\neq 0$.
We represent the causal structure among the variables with a DG $G=(V(G),E(G))$, in which $X_i\rightarrow X_j\in G$ if $X_i$ is a direct cause of $X_j$.  
Let $X\coloneqq[X_1\cdots X_p]^\top$. The model can be represented in matrix form as $X=B^\top X+N$, where $B$ is a $p\times p$ weighted adjacency matrix of $G$ 
with $B_{j,i}$ as its $(j,i)$-th entry
and $N=[N_1\cdots N_p]^\top$. Elements of $N$ are assumed to be jointly Gaussian and independent. Since we can always center the data, without loss of generality, we assume that $N$, and hence, $X$ is zero-mean. Therefore, $X\sim\mathcal{N}(0,\Sigma)$, where $\Sigma$ is the covariance matrix of the joint Gaussian distribution on $X$, and suffices to describe the distribution of $X$. We assume that $\Sigma$ is always invertible (the Lebesgue measure of non-invertible matrices is zero). Therefore, equivalently the precision matrix $\Theta=\Sigma^{-1}$ contains all the information regarding the distribution of $X$. $\Theta$ can be written as
%\vspace{-2mm}
\begin{equation}	
\Theta=(I-B)\Omega^{-1}(I-B)^{\top}, 
\label{eq:prec}
%\vspace{-2mm}
\end{equation}
where $\Omega$ is a $p\times p$ diagonal matrix with $\Omega_{i,i}=\sigma_i^2=\textit{Var}(N_i)$. In the sequel, we use the terms precision matrix and distribution interchangeably.

The most common notion of equivalence for DGs in the literature is Markov equivalence (also called independence equivalence) defined as follows:
%\vspace{-1mm}
\begin{definition}[Markov Equivalence]
\label{def:i-eq}
Let $\mathcal{I}(G)$ denote the set of all conditional d-separations\footnote{See \cite{pearl2009causality} for the definition of d-separation.} implied by the DG $G$.
DGs $G_1$ and $G_2$ are Markov equivalent if $\mathcal{I}(G_1)=\mathcal{I}(G_2)$.
\end{definition}
%\vspace{-2mm}
When cycles are permitted, defining equivalence of DGs based on CI relations that they represent is not suitable, as CI relations do not reflect all the information in the distribution that can be used for identification of the underlying structure; e.g., see \cite{lacerda2008discovering}.
That is, there exist DGs which can be distinguished using observational data with probability one despite representing the same CI relations.
We define the notion of equivalence based on the set of distributions which can be generated by a structure:
%\vspace{-1mm}
\begin{definition}[Distribution Set]
\label{def:distset}
 The distribution set of structure $G$, denoted by $\Theta(G)$, is  defined as
 \begin{equation*}
\begin{aligned}
\Theta(G)\!\coloneqq\! \{\Theta\!:\! \Theta &=(I-B)\Omega^{-1}(I-B)^{\top},\text{ for any }(B,\Omega)\\
&\text{ s.t. }\Omega\in \textit{diag}^+ \text{ and } \supp(B)\subseteq \supp(B_G)\},
\end{aligned}
\end{equation*}
where $\textit{diag}^+$ is the set of diagonal matrices with positive diagonal entries, $B_G$ is the binary adjacency matrix of $G$, and $\supp(B)=\{(i,j): B_{ij}\neq0\}$.
\end{definition}
%\vspace{-2mm}
$\Theta(G)$ is the set of all precision matrices (equivalently, distributions) that can be generated by $G$ for different choices of exogenous noise variances and edge weights in $G$. 
%\vspace{-1mm}
\begin{definition}[Distribution Equivalence]
\label{def:eq}
DGs $G_1$ and $G_2$ are distribution equivalent, or for short, equivalent, denoted by $G_1\equiv G_2$, if $\Theta(G_1)=\Theta(G_2)$.
\end{definition}
%\vspace{-2mm}
It is important to note that for DG $G$ and distribution $\Theta$, having $\Theta\in\Theta(G)$ does not imply that all the constraints of $\Theta$, such as its conditional independencies, can be read off of $G$. For instance, a complete DAG does not represent any conditional d-separations, yet all distributions are contained in its distribution set. This is due to the fact that the parameters in $B$ can be designed to represent certain extra constraints in the generated distribution. 

As mentioned earlier, we can have a pair of DGs which are distinguishable using observational data despite having the same conditional d-separations. This is not the case for DAGs. In fact, restricting the space of DGs to DAGs, Definitions \ref{def:eq} and \ref{def:i-eq} are equivalent.
%\vspace{-1mm}
\begin{proposition}
\label{prop:dag}
 Two DAGs $G_1$ and $G_2$ are equivalent if and only if they are Markov equivalent.
\end{proposition}
%\vspace{-2mm}
Therefore, one does not lose any information by caring only about Markov equivalence when dealing with acyclic structures.
All proofs are provided in the Supplementary Materials.

For general DGs, the graphical test for Markov equivalence is known to be significantly more complex \cite{richardson1996polynomial} than the test for DAGs \cite{verma1991equivalence}.
There are currently no known graphical conditions for distribution equivalence. This is the goal of Section \ref{sec:testgraphical}.

%\vspace{-3mm}
\section{Characterizing Equivalence}
\label{sec:test}
%\vspace{-2mm}

In order to determine whether DGs $G_1$ and $G_2$ are equivalent, a baseline equivalence test is as follows: We consider a distribution $\Theta\in\Theta(G_1)$ which results from a certain choice of parameters of $G_1$ in expression \eqref{eq:prec}, i.e., a certain choice of exogenous noise variances and edge weights. We then check whether there exists a choice of parameters for which $G_2$ generates $\Theta$. We then repeat the same procedure for $G_1$, considering $G_2$ as the original generator. 
More specifically, for DG $G_i$, let $Q_i=(I-B)\Omega^{-\frac{1}{2}}$ for any choice of $B$ such that $\supp(B)\subseteq \supp(B_{G_i})$ for $i\in\{1,2\}$.
For any choice of parameters of $G_1$ that results in distribution $\Theta=Q_1Q_1^\top$, we check if $Q_2Q_2^\top=\Theta$ has real-valued solution, and vice versa. 
Although this baseline equivalence test provides a systematic approach, it is tedious in many cases to check for the existence of a solution. In the following, we propose an alternative equivalence test based on rotations of $Q$.

Let $v_i$ be the $i$-th row of matrix $Q$. Therefore, $\Theta=QQ^\top$ is the Gramian matrix of the set of vectors $\{v_1,\cdots v_p\}$.
The set of generating vectors of a Gramian matrix can be determined up to isometry. That is, given $Q_1Q_1^\top=\Theta$, we have $Q_2Q_2^\top=\Theta$ if and only if $Q_2=Q_1U$ for some orthogonal transformation $U$. Therefore, $Q_1$ should be transformable to $Q_2$ by a rotation or an improper rotation (a rotation followed by a reflection).

In our problem of interest, for \emph{any} parameterization of $Q_1$ (resp. $Q_2$) it is necessary to check if there exists an orthogonal transformation of $Q_1$ (resp. $Q_2$) which can be generated for \emph{some} parameterization of $Q_2$ (resp. $Q_1$). Therefore, only the support of the matrix before and after the orthogonal transformation matters. Hence, we only need to consider rotation transformations. This can be formalized as follows: Let $Q_G$ be $B_G$ with 1s on its diagonal, i.e. $Q_G:= I+B_G$. This is the binary matrix that for all choices of parameters $B$ and $\Omega$, $\textit{supp}(Q)\subseteq \textit{supp}(Q_G)$.
%\vspace{-1mm}
\begin{proposition}
\label{prop:rot}
$G_1\equiv G_2$ if and only if for any choice of $Q_1$, there exists rotation $U^{(1)}$ such that $\supp(Q_{1}U^{(1)})\subseteq \supp(Q_{G_2})$, and for any choice of $Q_2$, there exists rotation $U^{(2)}$ such that $\supp(Q_{2}U^{(2)})\subseteq \supp(Q_{G_1})$.
\end{proposition}
%\vspace{-2mm}
To test the existence of a rotation required in Proposition \ref{prop:rot}, we propose utilizing a sequence of a special type of planar rotations called \emph{Givens rotations} \cite{golub2012matrix}.
%\vspace{-5mm}
\begin{definition}[Givens rotation] A Givens rotation is a rotation in the plane spanned by two coordinate axes. For a $\theta$-radian rotation in the $(j,k)$ plane, the entries of the Givens rotation matrix $G(j,k,\theta)=[g]_{p\times p}$ in $\mathbb{R}^p$ are 
$g_{i,i}=1$ for $i\not\in\{j,k\}$,
$g_{i,i}=\cos(\theta)$ for $i\in \{j,k\}$, and
$g_{k,j}\!=\!-g_{j,k}\!=\!-\sin(\theta)$,
and the rest of the entries are zero.
\end{definition}
%\vspace{-2mm}
Any rotation in $\mathbb{R}^p$ can be decomposed into a sequence of Givens rotations. Hence, in Proposition \ref{prop:rot}, we need to find a sequence of Givens matrices and define $U$ to be their product. The advantage of this approach is that the effect of a Givens rotation is easy to track: The effect of $G(j,k,\theta)$ on a row vector $v$ is as follows.
%\vspace{-1mm}
\begin{small}
\begin{equation}
\label{eq:Givens}
\begin{aligned}
  &[ v_1~ \cdots~ v_j ~\cdots~ v_k  \cdots~ v_p ]G(j,k,\theta)=\\
&
  [ v_1~ \cdots~ \cos(\theta)v_j\!+\!\sin(\theta)v_k  ~\cdots~  \!-\!\sin(\theta)v_j\!+\!\cos(\theta)v_k ~\cdots~ v_p ].
\end{aligned}
\end{equation}
\end{small}

\vspace{-5mm}
%\vspace{-9mm}
\subsection{Support Rotation}
%\vspace{-2mm}

As previously mentioned, since all choices of parameters in the structure need to be considered, it is necessary to determine the existence of a rotation that maps one support to another. We define support matrix and support rotation as follows.
%\vspace{-1mm}
\begin{definition}[Support matrix]
\label{def:supp}
For any matrix $Q$, its support matrix is a binary matrix $\xi$ of the same size with entries in $\{0,\times\}$, where $\xi_{i,j}=\times$ if $Q_{i,j}\neq0$ and $\xi_{i,j}=0$ otherwise. For directed graph $G$, we define its support matrix as support matrix of $Q_G$.
\end{definition}
%\vspace{-2mm}
Givens rotations can be used to introduce zeros in a matrix, and hence, change its support. Consider input matrix $Q$. Using expression \eqref{eq:Givens}, for any $i,j\in[p]$, $Q_{i,j}$ can be set to zero using a Givens rotation in the $(j,k)$ plane with angle $\theta=\tan^{-1}(-Q_{i,j}/Q_{i,k})$.
When zeroing $Q_{i,j}$, there may exist an index $l$ such that $Q_{l,j}$ or $Q_{l,k}$ will also become zero. However, since we consider all parameterizations of $Q$, we cannot take advantage of such accidental zeroings.
%\vspace{-1mm}
\begin{definition}[Support Rotation]
The support rotation $A(i,j,k)$ is a transformation that takes a support matrix $\xi$ as the input and sets $\xi_{i,j}$ to zero using a Givens rotation in the $(j,k)$ plane. The output is the support matrix of $QG(j,k,\tan^{-1}(-Q_{i,j}/Q_{i,k}))$, where $Q\in\arg\max_{Q'}|\supp(Q'G(j,k,\tan^{-1}(-Q'_{i,j} / Q'_{i,k})))|$ such that the support matrix of $Q'$ is $\xi$.
Note that $G(j,k,\tan^{-1}(-Q'_{ij}/Q'_{i,k}))$ is the Givens rotation in the $(j,k)$ plane which zeros $Q'_{i,j}$.
\end{definition}
%\vspace{-2mm}
Note that due to \eqref{eq:Givens}, $A(i,j,k)$ only affects the $j$-th and $k$-th columns of the input. The general effect of support rotation $A(i,j,k)$ is described in the following proposition.
%\vspace{-1mm}
\begin{proposition}
\label{prop:effect}
Support rotation $A(i,j,k)$ can have three possible effects on support matrix $\xi$:
%\vspace{-3mm}
\begin{enumerate}%[leftmargin=5mm,noitemsep]
\item If $\xi_{i,j}=0$, $A(i,j,k)$ has no effect.

\item If $\xi_{i,j}=\times$ and $\xi_{i,k}=\times$, $A(i,j,k)$ makes $\xi_{i,j}=0$, and for any $l\in[p]\setminus\{i\}$ such that
at least one of $\xi_{l,j}$ and $\xi_{l,k}$ is $\times$, $A(i,j,k)$ makes $\xi_{l,j}=\times$ and $\xi_{l,k}=\times$. 
This is obtained by an acute rotation.
\item If $\xi_{i,j}=\times$ and $\xi_{i,k}=0$, $A(i,j,k)$ switches columns $j$ and $k$ of $\xi$. This is obtained by a $\pi/2$ rotation.
\end{enumerate}
%\vspace{-3mm}
\end{proposition}
%\vspace{-2mm}
Figure \ref{fig:Aex} visualizes an example of a support rotation.
\begin{figure}[t]
\begin{center}
\includegraphics[scale=0.51]{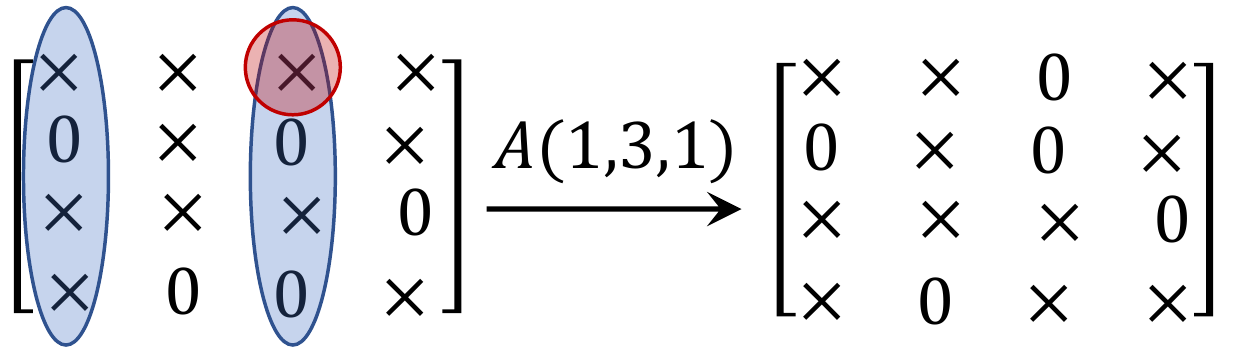}	
\end{center}
%\vspace{-5mm}
\caption{An example of support rotation (Case 2, Prop. \ref{prop:effect}). Element $\xi_{i,j}$ is in red, and columns $j$ and $k$ are in blue.}
%\vspace{-5mm}
\label{fig:Aex}
\end{figure}
Observe that the following four cases partition all the effects that can be obtained from a support rotation $A(i,j,k)$.
%\vspace{-3mm}
\begin{itemize}%[leftmargin=5mm,noitemsep]
\item {\bf Reduction.} If $\xi_{i,j}=\xi_{i,k}=\times$ and $\xi_{l,j}=\xi_{l,k}$ for all $l\in[p]\setminus\{i\}$, then only $\xi_{i,j}$ becomes zero.
\item {\bf Reversible acute rotation.} If $\xi_{i,j}=\xi_{i,k}=\times$ and there exists a row $i'$ such that the $j$-th and $k$-th columns differ only in that row, then $\xi_{i,j}$ becomes zero and both $\xi_{i',j}$ and $\xi_{i',k}$ become $\times$.
\item {\bf Irreversible acute rotation.} If $\xi_{i,j}=\xi_{i,k}=\times$ and the $j$-th and $k$-th columns differ in at least two rows, then $\xi_{i,j}$ becomes zero and all entries on the $j$-th and $k$-th columns become $\times$ on the rows on which they differed.
\item {\bf Column swap.} If $\xi_{i,j}=\times$ and $\xi_{i,k}=0$, then columns $j$ and $k$ are swapped.
\end{itemize}
%\vspace{-3mm}
Note that if $\xi$ is transformed to $\xi'$ via a reversible acute rotation $A(i,j,k)$, and $\xi_{i',j}=0$, then $\xi'$ can be mapped back to $\xi$ via $A(i',j,k)$, hence the name reversible.

%\vspace{-3mm}
\subsection{Characterizing Equivalence via Support Rotations}
%\vspace{-2mm}

We give the following necessary and sufficient condition for distribution equivalence of two structures using the introduced support operations. We show that irreversible acute rotations are not needed for checking equivalence. Here, for two support matrices $\xi$ and $\xi'$, we say $\xi\subseteq\xi'$ if $\supp(\xi)\subseteq\supp(\xi')$.
%\vspace{-1mm}
\begin{theorem}
\label{thm:mainrot2}
Let $\xi_1$ and $\xi_2$ be the support matrices of DGs $G_1$ and $G_2$, respectively.
$G_1$ is distribution equivalent to $G_2$ if and only if 
there exists a sequence of reductions, reversible acute rotations, and column swaps that maps $\xi_1$ to a subset of $\xi_2$, and a sequence that maps $\xi_2$ to a subset of $\xi_1$.
\end{theorem}
%\vspace{-2mm}
Theorem \ref{thm:mainrot2} converts the problem of determining the equivalence of two structures into a search problem for two sequences of support rotations. We propose to use a depth-first search algorithm that performs all column swaps at the end of the sequences. Due to space constraints, the pseudo-code is presented in the Supplementary Materials.

The following result is a nontrivial application of Theorem \ref{thm:mainrot2} regarding reversing cycles in DGs.
\begin{proposition}[Direction of Cycles]
\label{prop:dir}
Suppose structure $G_1$ contains a directed cycle $C$. Let $G_2$ be a structure that differs from $G_1$ in two ways. (1) The direction of cycle $C$ is reversed and (2) any variable pointing to $X_i\in C$ in $G_1$ via an edge which is not part of $C$ is, in $G_2$, pointing to the preceder of $X_i$ in $C$ in $G_1$. In this case, $G_1$ is distribution equivalent to $G_2$. (See Figure \ref{fig:prop5} for an example.)
\end{proposition}
%\vspace{-2mm}
\citet{richardson1996polynomial} presented a result similar to Proposition \ref{prop:dir} for the case of using CI relationships in the data and concluded that ``it is impossible to orient a cycle merely using CI information.'' Proposition \ref{prop:dir} extends that result by concluding that it is impossible to orient a cycle merely using observational data.

The following proposition provides a necessary and sufficient condition for equivalence for a specific class of DGs.
%\vspace{-1mm}
\begin{proposition}
\label{prop:col}
Consider DGs $G_1$ and $G_2$ with support matrices $\xi_1$ and $\xi_2$, respectively. If every pair of columns of $\xi_1$ differ in more than one entry, then $G_1\equiv G_2$ if and only if the columns of $\xi_2$ are a permutation of columns of $\xi_1$. 
\end{proposition}
\begin{figure}[t]
\begin{center}
\includegraphics[scale=0.41]{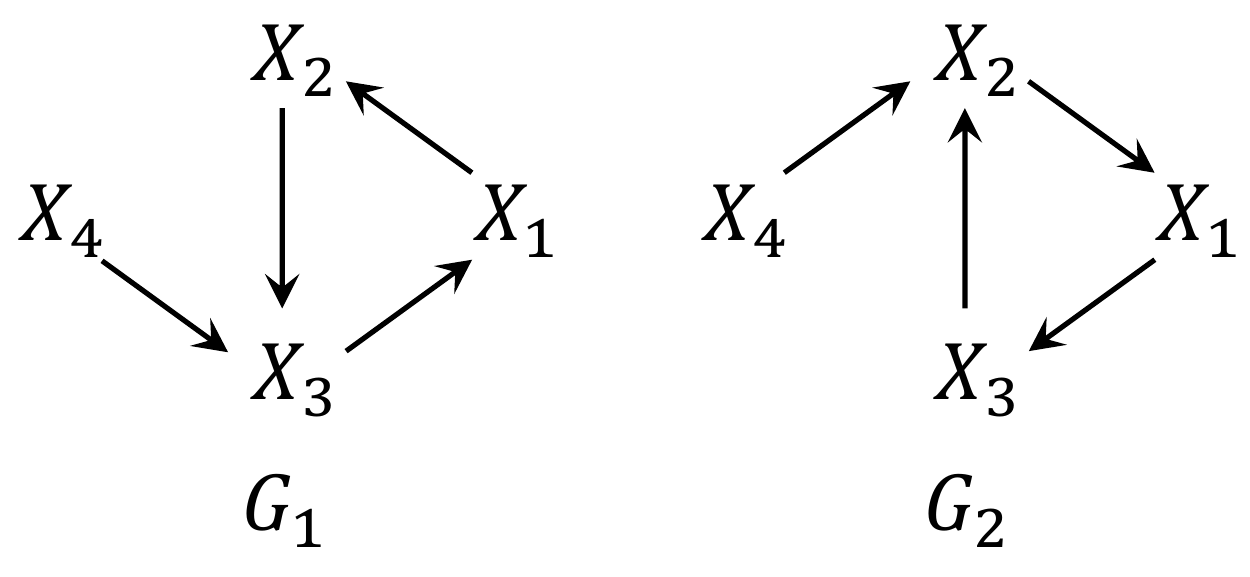}		
\end{center}	
%\vspace{-6mm}
\vspace{-2.5mm}
\caption{Example related to Proposition \ref{prop:dir}.}
%\vspace{-5mm}
\label{fig:prop5}
\end{figure}
%\vspace{-1mm}
\begin{example}
\label{ex:1}
In Figure \ref{fig:ex1}, (a) $G_1\equiv G_2$, (b) $G_1\not\equiv G_3$, and (c) $G_1\equiv G_4$.\\
(a) shows that unlike DAGs, equivalent DGs do not need to have the same skeleton or the same v-structures. 
To see $G_1\equiv G_2$, we note that
\scriptsize
\begin{align*}	
\xi_1=
\begin{bmatrix}
    \times\!\!\!\!\! & \times\!\!\!\!\! & \times \\
    0\!\!\!\!\! & \times\!\!\!\!\! & 0 \\
    0\!\!\!\!\! & \times\!\!\!\!\! & \times
\end{bmatrix}
\underrightarrow{~~A(1,3,1)~~}
\begin{bmatrix}
    \times\!\!\!\!\! & \times\!\!\!\!\! & 0 \\
    0\!\!\!\!\! & \times\!\!\!\!\! & 0 \\
    \times\!\!\!\!\! & \times\!\!\!\!\! & \times
\end{bmatrix}
\underrightarrow{~~A(3,1,2)~~}
\begin{bmatrix}
	\times\!\!\!\!\! & \times\!\!\!\!\! & 0 \\
    \times\!\!\!\!\! & \times\!\!\!\!\! & 0 \\
    0\!\!\!\!\! & \times\!\!\!\!\! & \times
\end{bmatrix}
\subseteq\xi_2.
\end{align*}
%\vspace{-4mm}
\begin{align*}	
\xi_2=
\begin{bmatrix}
    \times\!\!\!\!\! & \times\!\!\!\!\! & 0 \\
    \times\!\!\!\!\! & \times\!\!\!\!\! & 0 \\
    0\!\!\!\!\! & \times\!\!\!\!\! & \times
\end{bmatrix}
\underrightarrow{~~A(2,1,2)~~}
\begin{bmatrix}
    \times\!\!\!\!\! & \times\!\!\!\!\! & 0 \\
    0\!\!\!\!\! & \times\!\!\!\!\! & 0 \\
    \times\!\!\!\!\! & \times\!\!\!\!\! & \times
\end{bmatrix}
\underrightarrow{~~A(3,1,3)~~}
\begin{bmatrix}
    \times\!\!\!\!\! & \times\!\!\!\!\! & \times \\
    0\!\!\!\!\! & \times\!\!\!\!\! & 0 \\
    0\!\!\!\!\! & \times\!\!\!\!\! & \times
\end{bmatrix}
\subseteq\xi_1.
\end{align*}
\normalsize
$(b)$ follows from Proposition \ref{prop:col} since each pair of columns of $\xi_3$ differ in more than one entry. For $(c)$, we already have $\xi_1\subseteq\xi_4$. For the other direction,
%\vspace{-1mm}
\scriptsize
\begin{align*}	
\xi_4=
\begin{bmatrix}
    \times\!\!\!\!\! & \times\!\!\!\!\! & \times \\
    \times\!\!\!\!\! & \times\!\!\!\!\! & 0 \\
    0\!\!\!\!\! & \times\!\!\!\!\! & \times
\end{bmatrix}
\underrightarrow{~~A(2,1,2)~~}
\begin{bmatrix}
    \times\!\!\!\!\! & \times\!\!\!\!\! & \times \\
    0\!\!\!\!\! & \times\!\!\!\!\! & 0 \\
    \times\!\!\!\!\! & \times\!\!\!\!\! & \times
\end{bmatrix}
\underrightarrow{~~A(3,1,3)~~}
\begin{bmatrix}
    \times\!\!\!\!\! & \times\!\!\!\!\! & \times \\
    0\!\!\!\!\! & \times\!\!\!\!\! & 0 \\
    0\!\!\!\!\! & \times\!\!\!\!\! & \times
\end{bmatrix}
\subseteq\xi_1.
\end{align*}
\normalsize
\end{example}
%\vspace{-2mm}
As seen in Example \ref{ex:1}, structures $G_1$ and $G_4$ in Figure \ref{fig:ex1} are distribution equivalent. Therefore, the extra edge $X_2\rightarrow X_1$ in $G_4$ does not enable this structure to generate any additional distributions. In this case, we say structure $G_4$ is reducible. This idea is formalized as follows.
%\vspace{-1mm}
\begin{definition}[Reducibility]
	DG $G$ is reducible if there exists $G'$ such that $G\equiv G'$ and $E(G')\subset E(G)$. In this case, we say edges in $E(G)\setminus E(G')$ are reducible, and $G$ is reducible to $G'$.
\end{definition}
%\vspace{-2mm}
\begin{proposition}
\label{prop:red}
DG $G$ with support matrix $\xi$ is reducible if and only if there exists a sequence of reversible acute rotations that enables us to apply a reduction to $\xi$.	
\end{proposition}
%\vspace{-2mm}
Proposition \ref{prop:red} implies the following necessary condition for reducibility.
%\vspace{-1mm}
\begin{proposition}
\label{prop:2cycle}
A DG with no 2-cycles is {\it irreducible}.
\end{proposition}
%\vspace{-2mm}
A 2-cycle is a cycle over only two variables, such as the cycle over $X_1$ and $X_2$ in $G_2$ in Figure \ref{fig:ex1}.
Propositions \ref{prop:red} and \ref{prop:2cycle} lead to the following corollary regarding equivalence for DAGs, which bridges our proposed approach with the classic characterization for equivalence of DAGs.
%\vspace{-1mm}
\begin{corollary}
\label{cor:DAG}
DAGs $G_1$ and $G_2$ with support matrices $\xi_1$ and $\xi_2$ are equivalent if and only if 
there exists a sequence of reversible acute rotations and column swaps that maps $\xi_1$ to a subset of $\xi_2$, and one that maps $\xi_2$ to a subset of $\xi_1$.
\end{corollary}
\begin{figure}[t]
\centering
\includegraphics[scale=0.41]{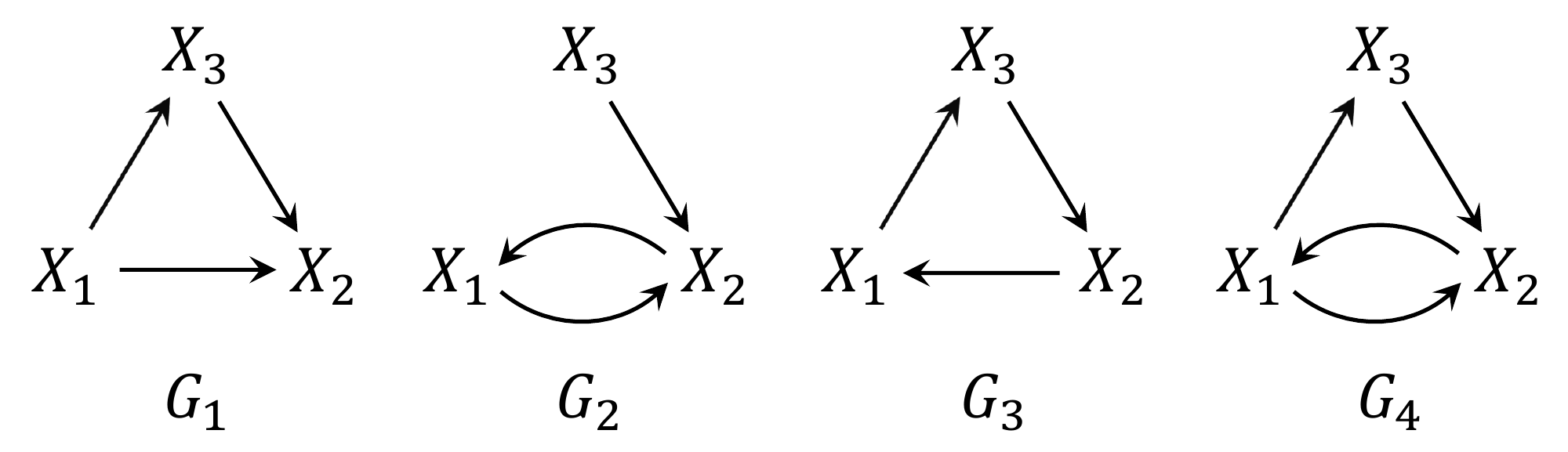}	
%\vspace{-5mm}
\vspace{-5mm}
\caption{DGs related to Example \ref{ex:1}.}
%\vspace{-5mm}
\label{fig:ex1}
\end{figure}
%\vspace{-0mm}

\begin{figure*}[t]
\begin{centering}
\includegraphics[scale=0.43]{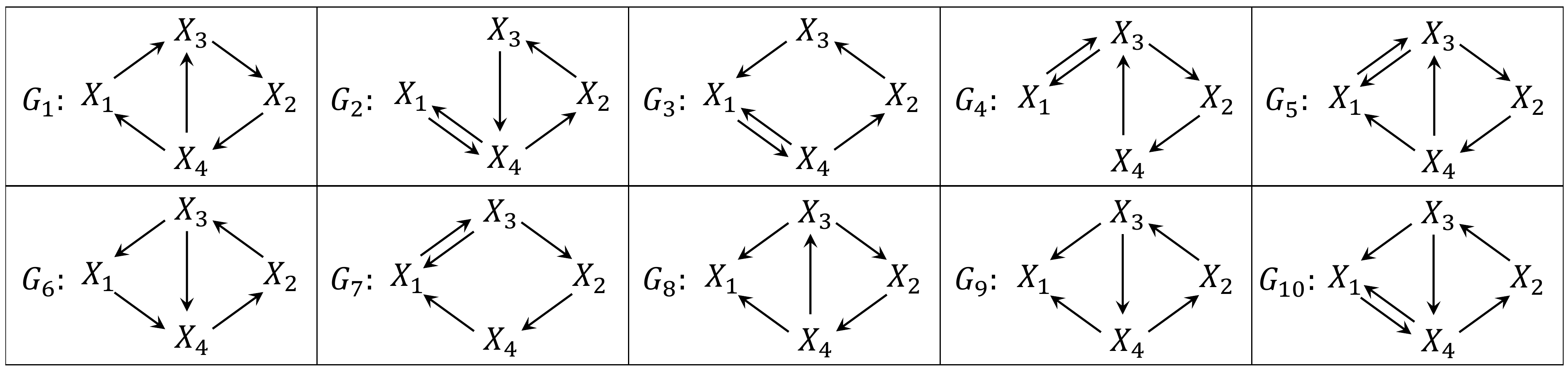}	
%\vspace{-3mm}
\caption{Elements of a distribution equivalence class.}
%\vspace{-5mm}
\label{fig:fex}
\end{centering}
\end{figure*}

%%%%%%%%%%%%%%%%%%%%%%%%%%%%%%%%%%%%%%%%%%%%%%%%%%%%%%%%%%%%%%%%

\begin{example}
\label{ex:chain}
We demonstrate our approach on a familiar equivalence example on DAGs:
Let $G_1:X_1\rightarrow X_2\rightarrow X_3$, $G_2:X_1\leftarrow X_2\leftarrow X_3$, and $G_3:X_1\rightarrow X_2\leftarrow X_3$.\\
(a) $G_1\equiv G_2$. (b) $G_1\not\equiv G_3$.
%%\vspace{-2mm}
%	\begin{align*}
%		&G_1:X_1\rightarrow X_2\rightarrow X_3\equiv G_2:X_1\leftarrow X_2\leftarrow X_3,\\ 
%		&G_1:X_1\rightarrow X_2\rightarrow X_3\not\equiv G_3:X_1\rightarrow X_2\leftarrow X_3.
%	\end{align*}
%%	\vskip -4mm

To see $G_1\equiv G_2$, we note that
\scriptsize
\begin{align*}	
\xi_1=
\begin{bmatrix}
    \times\!\!\!\!\! & \times\!\!\!\!\! & 0 \\
    0\!\!\!\!\! & \times\!\!\!\!\! & \times \\
    0\!\!\!\!\! & 0\!\!\!\!\! & \times
\end{bmatrix}
\underrightarrow{~~A(1,2,1)~~}
\begin{bmatrix}
    \times\!\!\!\!\! & 0\!\!\!\!\! & 0 \\
    \times\!\!\!\!\! & \times\!\!\!\!\! & \times \\
    0\!\!\!\!\! & 0\!\!\!\!\! & \times
\end{bmatrix}
\underrightarrow{~~A(2,3,2)~~}
\begin{bmatrix}
	\times\!\!\!\!\! & 0\!\!\!\!\! & 0 \\
    \times\!\!\!\!\! & \times\!\!\!\!\! & 0 \\
    0\!\!\!\!\! & \times\!\!\!\!\! & \times
\end{bmatrix}
\subseteq\xi_2.
\end{align*}
%\vspace{-4mm}
\begin{align*}	
\xi_2=
\begin{bmatrix}
    \times\!\!\!\!\! & 0\!\!\!\!\! & 0 \\
    \times\!\!\!\!\! & \times\!\!\!\!\! & 0 \\
    0\!\!\!\!\! & \times\!\!\!\!\! & \times
\end{bmatrix}
\underrightarrow{~~A(3,2,3)~~}
\begin{bmatrix}
    \times\!\!\!\!\! & 0\!\!\!\!\! & 0 \\
    \times\!\!\!\!\! & \times\!\!\!\!\! & \times \\
    0\!\!\!\!\! & 0\!\!\!\!\! & \times
\end{bmatrix}
\underrightarrow{~~A(2,1,2)~~}
\begin{bmatrix}
    \times\!\!\!\!\! & \times\!\!\!\!\! & 0 \\
    0\!\!\!\!\! & \times\!\!\!\!\! & \times \\
    0\!\!\!\!\! & 0\!\!\!\!\! & \times
\end{bmatrix}
\subseteq\xi_1.
\end{align*}
\normalsize
For the second part, we note that $\xi_3$ has two columns with two zeros, while $\xi_1$ has only one column with two zeros. Therefore, reversible acute rotations and column swaps cannot map $\xi_1$ to a subset of $\xi_3$. Therefore $G_1\not\equiv G_3$.
\end{example}

%%%%%%%%%%%%%%%%%%%%%%%%%%%%%%%%%%%%%%%%%%%%%%%%%%%%%%%%%%%%%%%%

%\vspace{-4mm}
\section{Graphical Characterization of Equivalence}
\label{sec:testgraphical}
%\vspace{-2mm}

In this section, we present a graphical counterpart to Theorem \ref{thm:mainrot2} by providing graphical counterparts to the rotations required by that Theorem.
%\vspace{-1mm}
\begin{definition}
For vertices $X_1$ and $X_2$, let $P_1:=\textit{Pa}(X_1)\cup\{X_1\}$ and $P_2:=\textit{Pa}(X_2)\cup\{X_2\}$, where $\textit{Pa}(X)$ denotes the set of parents of vertex $X$. $X_1$ and $X_2$ are parent reducible if $P_1=P_2$ and parent exchangeable if $|P_1\triangle P_2|=1$, where $\triangle$ is the symmetric difference operator, which identifies elements which are only in one of the sets. 
\end{definition}
%\vspace{-2mm}
The three rotations in Theorem \ref{thm:mainrot2} lead to the following graphical operations:
%\vspace{-4mm}
\begin{itemize}%[leftmargin=5mm,noitemsep]
\item {\bf Parent reduction.} If $X_j$ and $X_k$ are parent reducible, any support rotation on columns $\xi_{\cdot,j}$ and $\xi_{\cdot,k}$ which zeros a non-zero entry on those columns except $\xi_{j,j}$ and $\xi_{k,k}$ removes the parent from $X_j$ or $X_k$ corresponding to the zeroed entry. We call this edge removal a parent reduction. The support rotation in this case is of reduction rotation type.
\item {\bf Parent exchange.} If $X_j$ and $X_k$ are parent exchangeable, by definition there exists $X_i$ such that $P_j\triangle P_k=\{X_{i}\}$. In this case, any support rotation on columns $\xi_{\cdot,j}$ and $\xi_{\cdot,k}$ which zeros a non-zero entry on those columns except $\xi_{j,j}$ and $\xi_{k,k}$ removes the parent from $X_j$ or $X_k$ corresponding to the zeroed entry. Additionally, the missing edge from $X_{i}$ to $X_j$ or $X_k$ is added. We call this a parent exchange. The support rotation in this case is of column swap or reversible acute rotation type.
\item {\bf Cycle reversion.} A cycle reversion swaps the column of each member of a cycle $C$ with the column corresponding to its preceder in the cycle. This reverses the direction of the cycle $C$ and changes any edge outside of $C$ connecting to an $X_i\in C$ in the original DG to point instead to the preceder of $X_i$ in $C$.
\end{itemize}
%\vspace{-3mm}
Note that in the graphical operations above, we exclude support rotations that lead to zeroing a diagonal entry, since they do not have a graphical representation (by Def. \ref{def:supp}).

Equipped with the graphical operations,  
we present a graphical counterpart
to Theorem \ref{thm:mainrot2}.
%\vspace{-1mm}
\begin{theorem}
\label{thm:graphical}
$G_1$ is distribution equivalent to $G_2$ if and only if
there exists a sequence of parent reductions, parent exchanges, and cycle reversions that maps $G_1$ to a subgraph of $G_2$, and a sequence that maps $G_2$ to a subgraph of $G_1$.
\end{theorem}
%\vspace{-2mm}
\begin{example}
Figure \ref{fig:fex} shows the elements of a distribution equivalence class. Suppose $G_1$ is the original structure. Cycle reversion on the cycle $(X_2,X_4,X_3,X_2)$ results in $G_2$, cycle reversion on the cycle $(X_1,X_3,X_2,X_4,X_1)$	results in $G_3$, parent exchange $A(4,1,3)$ results in $G_4$, and parent exchange $A(1,3,1)$ results in $G_8$.
\end{example}
%\vspace{-2mm}
\begin{remark}
Given observational data from any of the structures in Figure \ref{fig:fex}, CI-based structure learning methods such as CCD \cite{richardson1996discovery} may output a structure (for example $G_1$ without edges $X_4\rightarrow X_1$) which is not distribution equivalent to the ground truth. This can be prevented by leveraging other statistical information in the distribution beyond CI relationships.
\end{remark}
%\vspace{-2mm}
We have the following corollary regarding equivalence for DAGs. The reasoning is the same as in Corollary \ref{cor:DAG}.
%\vspace{-1mm}
\begin{corollary}
\label{cor:DAGgraphical}
DAGs $G_1$ and $G_2$ are equivalent if and only if 
there exists a sequence of parent exchanges
that maps $G_1$ to $G_2$, and one that maps $G_2$ to $G_1$.
\end{corollary}

%\vspace{-5mm}
\section{Learning Directed Graphs from Data}
\label{sec:learning}
%\vspace{-2mm}

Structure $G$ imposes constraints on the entries of precision matrix $\Theta$. We will refer to such constraints as the \emph{distributional constraints} of $G$. Every distribution in $\Theta(G)$ should satisfy the distributional constraints of $G$.
Clearly, two DGs are distribution equivalent if and only if they have the same distributional constraints.
We call a distributional constraint a \emph{hard constraint} if the set of the values satisfying that constraint is Lebesgue measure zero over the space of the parameters involved in the constraint. For instance in DAGs, if $X_i$ and $X_j$ are non-adjacent and have no common children, we have the hard constraint $\Theta_{i,j}=0$. We denote the set of hard constraints of a DG $G$ by $H(G)$.

Recall that distribution equivalence of two structures $G_1$ and $G_2$ implies that any distribution that can be generated by $G_1$ can also be generated by $G_2$, and vice versa. Therefore, no distribution can help us distinguish between $G_1$ and $G_2$. However, in practice we usually have access to only one distribution which is generated from a ground truth structure, and it may be the case that this distribution can be generated by another structure which is not equivalent to the ground truth.
Therefore, finding the distribution equivalence class of the ground truth structure from one distribution is in general not possible, and extra considerations are required for the problem to be well defined. Below we will accordingly provide a weaker notion of equivalence and show that the ground truth can be recovered up to this equivalence.

The aforementioned issue also arises when learning DAGs and considering Markov equivalence. The most common approach to dealing with this issue in the literature is to assume that the distribution is \emph{faithful} to the ground truth structure. This requires a one-to-one correspondence between the conditional d-separations of the ground truth structure and the CI relationships in the distribution \cite{spirtes2000causation}. This is a sensible assumption from the perspective that the Lebesgue measure of the parameters which lead to 
extra CIs in the generated distribution is zero \cite{meek2013strong}. 

The case of general DGs is more complex since they can require other distributional constraints besides CIs. In particular, we may have distributional constraints 
other than hard constraints due to cycles.
Hence, in this case the Lebesgue measure of the parameters which lead to extra distributional constraints in the generated distribution is not necessarily zero. This motivates the following weaker notion of equivalence for  structure learning from observational data.
%\vspace{-1mm} 
\begin{definition}[Quasi Equivalence]
\label{def:q-eq}
Let $\theta_G$ be the set of linearly independent parameters needed to parameterize any distribution $\Theta\in\Theta(G)$.
For two DGs $G_1$ and $G_2$, let $\mu$ be the Lebesgue measure defined over $\theta_{G_1}\cup\theta_{G_2}$.  
$G_1$ and $G_2$ are quasi equivalent, denoted by $G_1\cong G_2$, if  $\mu(\theta_{G_1}\cap\theta_{G_2})\ne0$.
\end{definition}
%\vspace{-2mm}
Roughly speaking, two DGs are quasi equivalent if the 
set of distributions that they can both generate
has a non-zero Lebesgue measure.
Note that
Definition \ref{def:q-eq} implies that if
DGs $G_1$ and $G_2$ are quasi equivalent they share the same hard constraints.
We have the following assumption for structure learning, which is a generalization of faithfulness:
%\vspace{-1mm} 
\begin{definition}[Generalized faithfulness]
\label{def:G_faith}	
A distribution $\Theta$ is generalized faithful (g-faithful) to structure $G$ if $\Theta$ satisfies a hard constraint $\kappa$ if and only if $\kappa\in H(G)$. 
\end{definition}
%\vspace{-2mm}
\begin{assumption}
\label{ass:G_faith}
	The generated distribution is g-faithful to the ground truth structure $G^*$, and for irreducible DG $G^*$, if there exists a DG $G$ such that $H(G)\subseteq H(G^*)$ and $|E(G)|\le|E(G^*)|$, then $H(G)= H(G^*)$.
\end{assumption}
%\vspace{-2mm}
The following justifies the first part of Assumption \ref{ass:G_faith}:
%\vspace{-1mm}
\begin{proposition}
\label{prop:mes0}
With respect to Lebesgue measure over $\theta_{G}$, the set of distributions not g-faithful to $G$ is measure zero.
\end{proposition}
%\vspace{-2mm}
The second part of Assumption \ref{ass:G_faith} requires that if the ground truth structure $G^*$ has no reducible edges and there exists another DG $G$ that has only relaxed some of the hard constraints of $G^*$, then $G$ must have more edges than $G^*$. This is clearly the case for DAGs.
%\vspace{-1mm}
\begin{proposition}
\label{prop:consistent}
Under Assumption \ref{ass:G_faith}, quasi equivalence is the extent of identifiability from observational data.
\end{proposition}
%\vspace{-2mm}

%\vspace{-3mm}
\subsection{Score-Based Structure Learning}
%\vspace{-2mm}

We propose a score-based method for structure learning based on local search.
Score-based methods are well-established in the literature for learning DAGs. The predominant approach is to maximize the regularized likelihood of the data by performing a greedy search over all DAGs \cite{heckerman1995learning}, equivalence classes of DAGs \cite{chickering2002optimal}, or permutations of the variables \cite{teyssier2012ordering,solus2017consistency}.
Also, works such as
\cite{van2013ell_,fu2013learning,aragam2015concave,raskutti2018learning,zheng2018dags} specifically consider the problem of learning a linear Gaussian acyclic model via penalized parameter estimation.

To the best of our knowledge, there are no existing score-based structure learning approaches for the cyclic linear Gaussian model.
In light of our theory, we propose to use the $\ell_0$-regularized negative log likelihood function as the score, which is a standard choice of the score in the literature of learning DAGs, and show that it is able to recover the quasi equivalence class of the underlying DG.
Let $\mathbf{X}$ be the $n\times p$ data matrix. The $\ell_0$-regularized ML estimator solves the following unconstrained optimization problem:
\begin{equation}
\label{eq:l0ML}
%\vspace{-2mm}
\min_G\min_{\substack{(B,\Omega): \supp(B)\subseteq\supp(B_G)}}\mathcal{L}(\mathbf{X}:B,\Omega)+\lambda\|B\|_0,
\end{equation}
 where $\mathcal{L}(\mathbf{X}\!:\!B,\Omega)\!=\!-n\log(\det(I\!-\!B))\!+\!\sum_{i=1}^p\!\frac{n}{2}\!\log(\sigma_i^2)+\frac{1}{2\sigma_i^2}\|\mathbf{X}_{\cdot,i}\!-\!\mathbf{X}B_{\cdot,i}\|_2^2$
is the negative log-likelihood of the data, $\|B\|_0:=\sum_{i,j}\mathds{1}_{x\neq0}(B_{i,j})$, and similar to the BIC score, we set $\lambda=0.5\log n$.
%\vspace{-1mm}
\begin{remark}
\label{rmk:irr}
The estimator in \eqref{eq:l0ML} will never output a reducible DG, since removing redundant edges improves the score. This is in line with the minimality assumption in the literature for DAGs \cite{pearl1988probabilistic,raskutti2018learning}.
\end{remark}
%\vspace{-2mm}
\begin{theorem}
\label{thm:l0}
Under Assumption \ref{ass:G_faith}, the
global minimizer of \eqref{eq:l0ML}
with $\lambda=0.5\log n$ outputs $\hat{G}\cong G^*$ asymptotically.
\end{theorem}
%\vspace{-2mm}
Hence, by Prop. \ref{prop:consistent} and Theorem \ref{thm:l0}, the score \eqref{eq:l0ML} is consistent, i.e., it asymptotically achieves the extent of identifiability. 

%\vspace{-3mm}
\subsubsection{Structure Search}
%\vspace{-2mm}

We solve the outer optimization problem in \eqref{eq:l0ML} via local search over the structures. We choose the search space to contain all DGs and use the standard operators (i.e., local changes) of edge addition, deletion, and reversal. See \cite{koller2009probabilistic} for a discussion regarding the necessity of these operators. Two main issues arise when cycles are allowed in the structure:

{\bf Virtual edges.}
There exists a virtual edge between non-adjacent vertices $X_i$ and $X_j$ if they have a common child $X_k$ which is an ancestor of $X_i$ or $X_j$ \cite{richardson1996polynomial}.
If a greedy search algorithm does not find $X_k$ and $X_i$ (or $X_j$) to be on a cycle, it can significantly increase the likelihood by adding an edge at the location of the virtual edge. The algorithm would therefore be trapped in a local optimum with one more edge than the ground truth. To resolve this issue, we propose adding the following fourth search operator: 
Suppose we have a triangle over three variables $X_i$, $X_j$ and $X_k$, and there exists an additional sequence of edges connecting $X_j$ and $X_k$. In one atomic move, we perform a series of edge reversals to form a cycle containing $X_j\to X_k$ along the sequence, delete the edge connecting $X_i$ to $X_j$, and orient the edge $X_i\to X_k$. If the likelihood is unchanged, the edge deletion improves the score. In the case that the oriented cycle is of length two, additional considerations are needed; see the Supplementary Materials for details as well as simulations justifying this fourth operator.

{\bf Score decomposability.}
When the DG is acyclic, the distribution generated by a linear Gaussian structural equation model satisfies the local Markov property. This implies that the joint distribution can be factorized into the product of the distributions of the variables conditioned on their parents. The benefit of this factorization is that the computational complexity of evaluating the effect of operators can be dramatically reduced since a local change in the structure does not change the score of other parts of the DAG. In contrast, for the case of cyclic DGs the distribution does not necessarily satisfy the local Markov property.
However, the distribution still satisfies the global Markov property \cite{spirtes1995directed}. Therefore, our search procedure factorizes the joint distribution into the product of conditional distributions. Each of these distributions is over the variables in a maximal strongly connected subgraph (MSCS), conditioned on their parents outside of the MSCS. After applying an operation, the likelihoods of all involved MSCSs are updated; see the Supplementary Materials for additional details.

The implementation of the approach is publicly available at \href{https://github.com/syanga/dglearn}{https://github.com/syanga/dglearn}.

%\vspace{-4mm}
\vspace{-2mm}
\section{Experiments}
\label{sec:exp}
\vspace{-1mm}
%\vspace{-2mm}
%\vspace{7mm}

We generated $100$ random ground truth DGs of orders $p\in\{5,20,50\}$, all with maximum degree $4$. The DGs are constrained to have maximum cycle lengths $5$, $5$, and $10$, respectively. For each structure, we sampled the edge weights uniformly from $B_{i,j}\in[-0.8,-0.2]\cup[0.2,0.8]$ and the exogenous noise variances uniformly from $\sigma^2_{i}\in[1,3]$ to generate the data matrix $\mathbf{X}$ of size $10^4\times p$. We constrained the ground truth $B$ matrices to be stable via an accept-reject approach; the modulus of all eigenvalues of $B$ should be strictly less than one. The stability of a model guarantees that the effects of one-time noise dissipate. Our search algorithms were also constrained to only output stable structures. We used the following standard local search methods: 1. Hill climbing 2. Tabu search \cite{koller2009probabilistic}.

Evaluating the performance of a learning approach is not trivial for the case of general DGs. 
As seen before, equivalent cyclic DGs may have very different skeletons. Hence, conventional evaluation metrics such as structural Hamming distance (SHD) with the ground truth DG or comparison of the learned and ground truth adjacency matrices cannot be used. We propose the following evaluation methods:

{\bf 1. SHD Evaluation.}
We enumerate the set of all DGs equivalent to the ground truth DG using Algorithm 1 in the Supplementary Materials to form the distribution equivalence class of the ground truth. We then compute the smallest SHD between the algorithm's output DG and the members of the equivalence class as a measure of the performance.

{\bf 2. Multi-Domain Evaluation.}
Suppose the input data is sampled from a distribution $\Theta$ generated by ground truth DG $G^*$, and let $\hat{G}$ denote an algorithm's output structure.
Due to finite sample size and the possible violation of Assumption \ref{ass:G_faith}, $\hat{G}$ may be able to maximize the likelihood yet not be (quasi) equivalent to $G^*$.
In general, we expect such an output to be compatible with only the given data and not with data sampled from other distributions generated by $G^*$. We therefore propose the following evaluation approach.
%\vspace{-5mm}
\begin{enumerate}%[leftmargin=5mm,noitemsep]
\item For ground truth structure $G^*$, generate $d$ distributions $\{\Theta_{1},...,\Theta_{d}\}$ by sampling edge weights and variances.
\item For each $\Theta_{i}$, run the algorithm to obtain $\hat{G}_{i}$.
\item For each $\hat{G}_{i}$, optimize its edge weights and variances to generate distributions $\{\hat{\Theta}_{i,1},...,\hat{\Theta}_{i,d}\}$ such that $\hat{\Theta}_{i,j}$ minimizes the KL-divergence to
$\Theta_{j}\in\{\Theta_{1},...,\Theta_{d}\}$.
\item The success rate of $\hat{G}_{i}$ is the percentage of domains for which the minimizing KL-divergence computed in step 3 is below a threshold $\eta$.
\end{enumerate}

%\vspace{-1mm}
\begin{figure}[t]
\centering
\begin{minipage}{.25\textwidth}
  \centering
  \includegraphics[width=.99\linewidth,height=33mm]{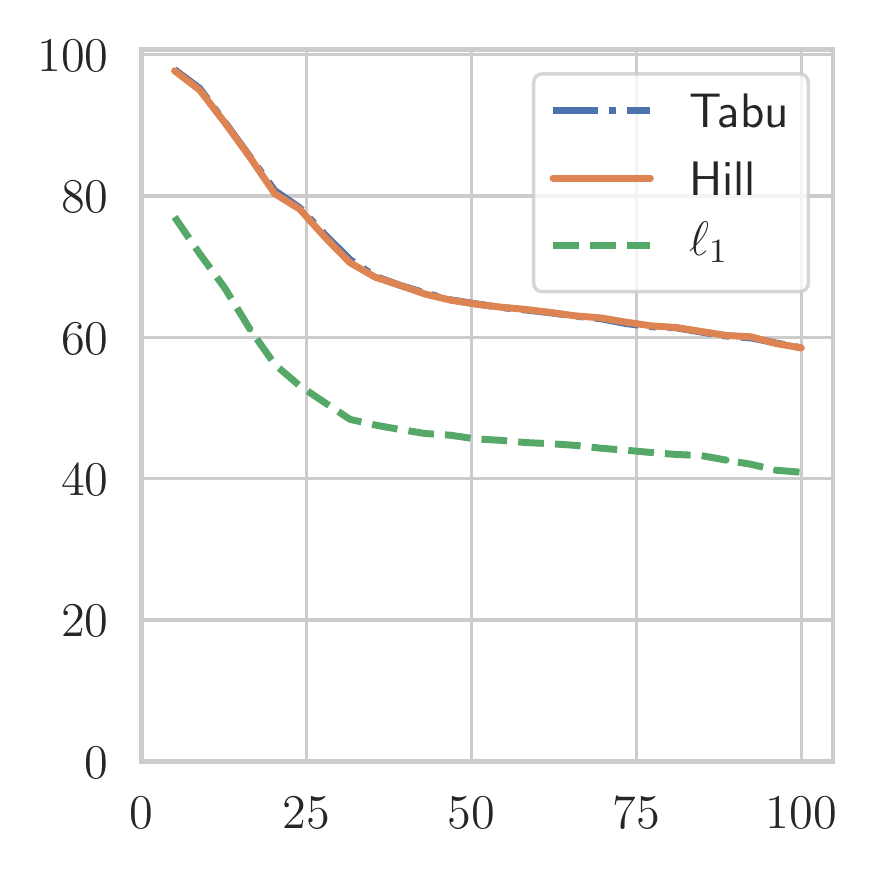}
  %\captionof{figure}{A figure}
  %\label{fig:test1}
\end{minipage}%
\begin{minipage}{.25\textwidth}
  \centering
  \includegraphics[width=.99\linewidth,height=33mm]{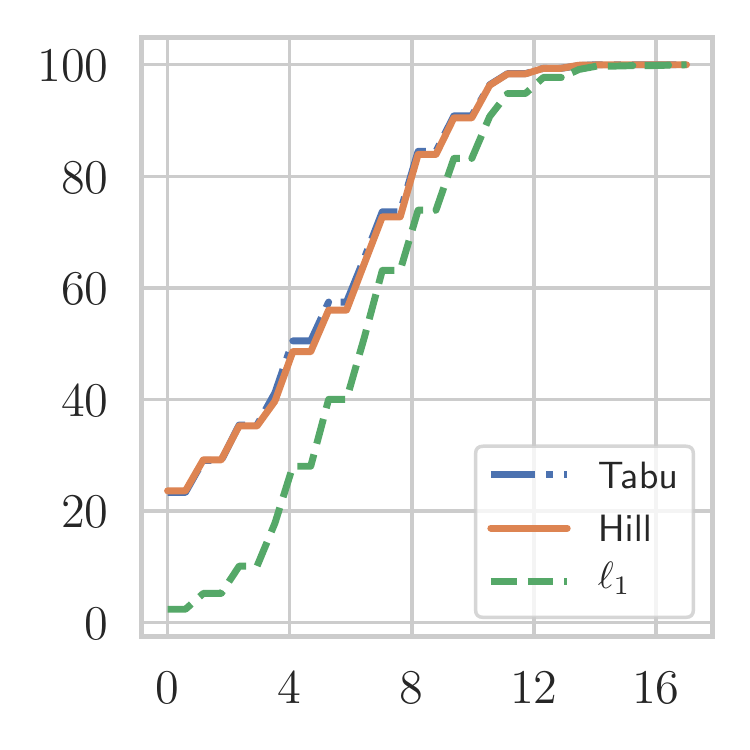}
  %\captionof{figure}{Another figure}
  %\label{fig:test2}
\end{minipage}\vspace{0mm}
\begin{minipage}{.25\textwidth}
  \centering
  \includegraphics[width=.99\linewidth,height=33mm]{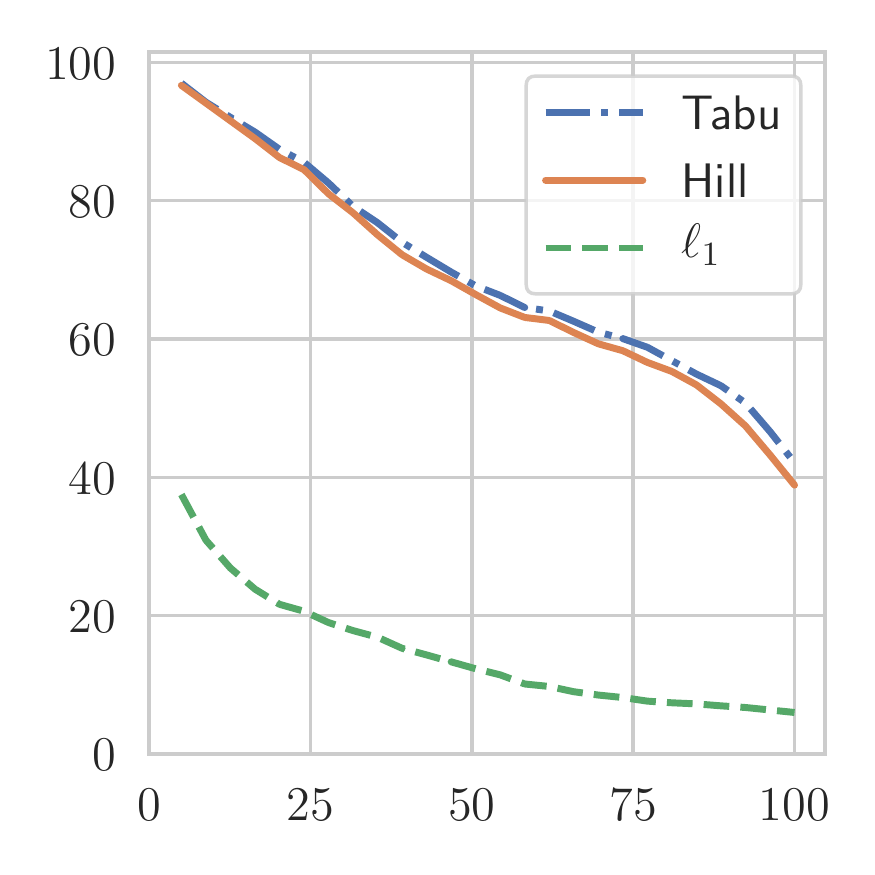}
  %\captionof{figure}{A figure}
  %\label{fig:test1}
\end{minipage}%
\begin{minipage}{.25\textwidth}
  \centering
  \includegraphics[width=.99\linewidth,height=33mm]{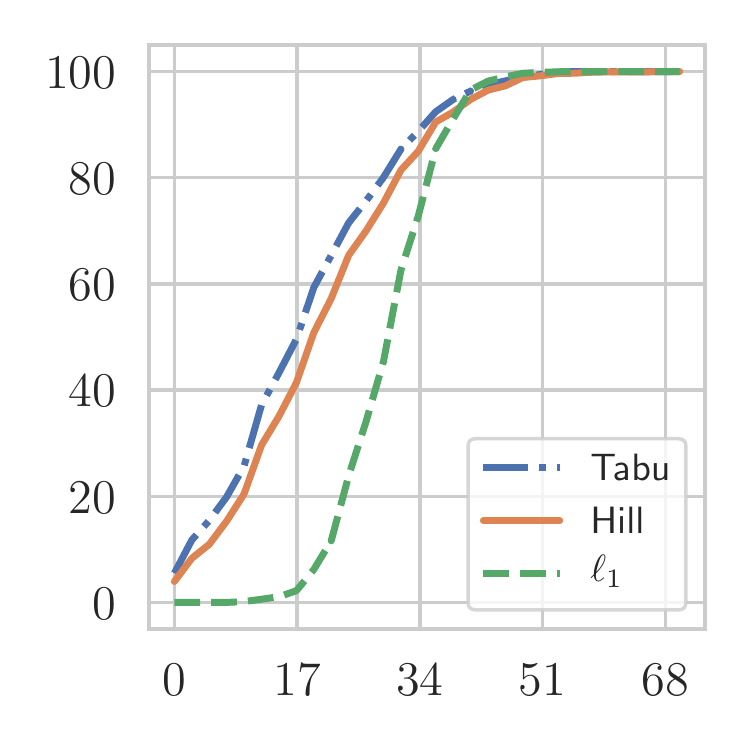}
  %\captionof{figure}{Another figure}
  %\label{fig:test2}
\end{minipage}\vspace{0mm}
\begin{minipage}{.25\textwidth}
  \centering
  \includegraphics[width=.99\linewidth,height=33mm]{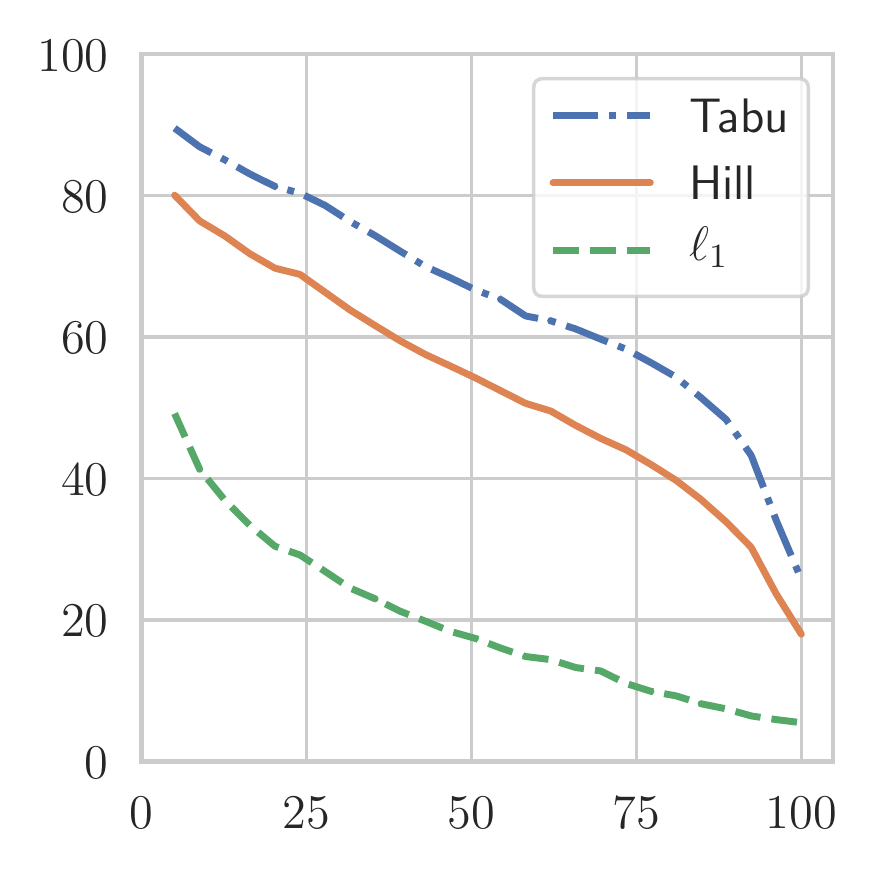}
  %\captionof{figure}{A figure}
  %\label{fig:test1}
\end{minipage}%
\begin{minipage}{.25\textwidth}
  \centering
  \includegraphics[width=.99\linewidth,height=33mm]{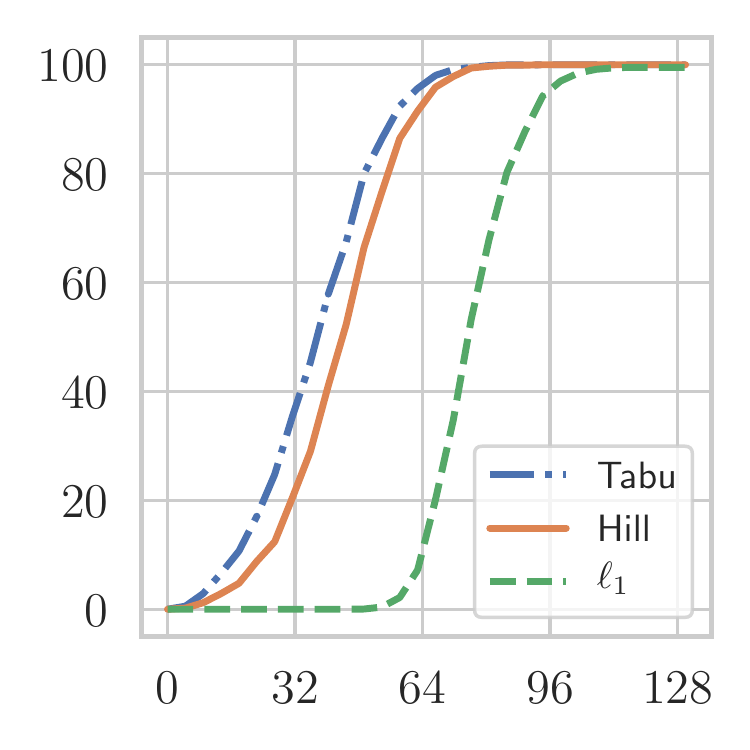}
  %\captionof{figure}{Another figure}
  %\label{fig:test2}
\end{minipage}
%\caption{}
%\vspace{-5mm}
\caption{Results for $p=5,20,50$, top to bottom. {\bf Left column:} multi-domain evaluation. The percentage of outputs with success rate larger than a certain value is plotted vs. success percentages; e.g., for $p=20$, $80\%$ of the outputs could generate more than $25\%$ of the distributions generated by their corresponding ground truth. {\bf Right column:} SHD evaluation. The percentage of outputs with SHD less than or equal to a certain value is plotted vs. SHD.} 
%\vspace{-4mm}
\label{fig:exps}
\end{figure}

%\vspace{-5mm}
Since domain distributions are generated randomly, if the success rate of output $\hat{G}_{i}$ is large, there is a non-negligible subset of the distribution set of $G^*$ that $\hat{G}_{i}$ can generate as well. Hence, $\hat{G}_{i}$ is quasi equivalent to $G^*$.
In our evaluations, we used $d=50$ and $\eta=p\times10^{-3}$.
We emphasize that multi-domain data is \emph{only} used for evaluation. In the learning stage, only one distribution is used.

We cannot compare the performance of our approach with the performance of methods based on CI relationships (such as CCD), since those approaches return a PAG representing all Markov equivalent DGs, which usually represents a much larger set of DGs than the distribution equivalence class. We therefore only compared our approach with an $\ell_1$-regularized maximum likelihood estimator which directly solves the optimization problem
$\min_{B,\Omega}\mathcal{L}(\mathbf{X}:B,\Omega)+\lambda\|B\|_1$, which does not need a separate structure search.
The results are given in Figure \ref{fig:exps}. The figure shows that our proposed approach successfully finds DGs capable of generating distributions generated by the ground truth structure.
While the SHD evaluation shows that the outputs are not always distribution equivalent, the multi-domain evaluation provides evidence that many are quasi equivalent to the ground truth.
We also evaluated the effect of sample size on the performance in the Supplementary Materials.

%\vspace{-3mm}
\subsection{fMRI hippocampus data}
%\vspace{-2mm}
We considered the fMRI hippocampus dataset \cite{fMRI}, which contains signals from six separate brain regions: perirhinal cortex (PRC), parahippocampal cortex
(PHC), entorhinal cortex (ERC), subiculum (Sub), CA1, and CA3/Dentate Gyrus (CA3) in the resting state.
We used the anatomical connections  \cite{bird2008hippocampus,zhang2015discovery} as the ground truth, depicted in Figure \ref{fig:fmri}.
We applied our proposed method on one of the domains in the dataset and found that two out of eight structures equivalent to the ground truth were (local) optima for the score even though there is no evidence that the data are linear Gaussian.

\begin{figure}[t]
\begin{center}
\includegraphics[scale=0.48]{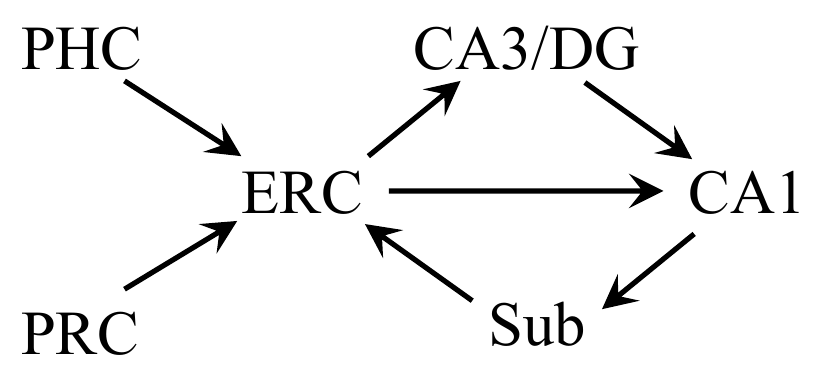}		
\end{center}	
\caption{Ground truth structure for the fMRI hippocampus dataset.}
\vspace{-2mm}
\label{fig:fmri}
\end{figure}

%\vspace{-4mm}
\section{Conclusion}
\label{sec:conc}
%\vspace{-2mm}

We presented a general, unified notion of equivalence for linear Gaussian DGs and proposed methods for characterizing the equivalence of two structures. We also proposed a score-based structure learning approach that asymptotically achieves the extent of identifiability. Our results are instrumental to the fields of causality and graphical models. 
From the causality perspective, consider for example Figure \ref{fig:fex}. Our results guarantee a direct causal effect between $X_2$ and $X_4$ and show that a direct causal effect does not necessarily exist between $X_3$ and $X_4$. From the graphical models perspective, our results provide the tools to handle distributions that lack a DAG representation but can be modeled by a cyclic DG. We hope that this work spurs further research in the study of directed graphs.

\newpage

\section*{Acknowledgements}

This work was supported in part by ONR grant W911NF-15-1-0479, NSF CCF 1704970, and NSF CNS 16-24811. 
KZ would like to acknowledge the support by National Institutes of Health under Contract No. NIH-1R01EB022858-01, FAIN-R01EB022858, NIH-1R01LM012087, NIH5U54HG008540-02, and FAIN-U54HG008540, and by the United States Air Force under Contract No. FA8650-17-C7715.

\bibliography{Refs.bib}
\bibliographystyle{icml2020}

\setlength{\parskip}{2mm}
\setlength{\parindent}{0cm}
\newpage
\onecolumn

\newpage

 \fontsize{10}{15}\selectfont
~\vspace{1cm}
\begin{center}
{\LARGE \bf 
Supplementary Materials}
\end{center}

\vskip 0.7in

\appendix

\section{Proof of Proposition \ref{prop:dag}}

Two DAGs are Markov equivalent if and only if they have the same skeleton and v-structures \citep{verma1991equivalence}. 
Therefore, it suffices to show that two DAGs $G_1$ and $G_2$ are distribution equivalent if and only if they have the same skeleton and v-structures.

By Corollary \ref{cor:DAGgraphical}, DAGs $G_1$ and $G_2$ are equivalent if and only if there exist sequences of parent exchanges that map them to one another.
Suppose $G_1$ and $G_2$ are distribution equivalent. Therefore there exists a sequence of parent exchanges mapping one to another. Since DAGs do not have 2-cycles, parent exchange for them will only result in flipping an edge, and since the other parents of the vertices at the two ends of that edge should be the same, it does not generate or remove a v-structure. Therefore, the sequence of parent exchanges does not change the skeleton or change the set of v-structures. Therefore, $G_1$ and $G_2$ are Markov equivalent.

If two DAGs $G_1$ and $G_2$ have the same skeleton and v-structures, then their difference can be demonstrated as a sequence of edge flips such that in each flip, all the parent of the two ends have been the same, which means this flip is a parent exchange. Therefore, by Corollary \ref{cor:DAGgraphical}, DAGs $G_1$ and $G_2$ are distribution equivalent.

\section{Proof of Proposition \ref{prop:rot}}

If side:\\
If $\supp(Q_1U^{(1)})\subseteq\supp(Q_{G_2})$, then we can simply choose the entries of $Q_1U^{(1)}$ as the entries  of $Q_2$ (as they are all free variables). Therefore, 
\[
Q_2Q_2^\top=Q_1U^{(1)}(U^{(1)})^\top Q_1^\top=Q_1Q_1^\top.
\]
That is, $Q_2$ can generate the distribution which was generated by $Q_1$. Since this is true for all choices of $Q_1$, and since the reverse (i.e., starting with $Q_2$) is also true, by definition, $G_1$ is distribution equivalent to $G_2$.

Only if side:\\
If $G_1$ is distribution equivalent to $G_2$, then for all choices of $Q_1$, generating $Q_1Q_1^\top=\Theta$, there exists $Q_2$ generated by $G_2$, such that $Q_2Q_2^\top=\Theta$. Since $Q_2$ is generated by $G_2$, by definition, $\supp(Q_2)\subseteq\supp(Q_{G_2})$. Also, since $Q_1Q_1^\top=\Theta$ and $Q_2Q_2^\top=\Theta$, we have $Q_2=Q_1U$, for some orthogonal transformation $U$, due to the fact that the generating vectors of a Gramian matrix can be determined up to isometry. Therefore, since $Q_2=Q_1U$ and $\supp(Q_2)\subseteq\supp(Q_{G_2})$, we conclude that $\supp(Q_1U)\subseteq\supp(Q_{G_2})$. It remains to show that there exists a rotation $U^{(1)}$, for which $\supp(Q_1U^{(1)})\subseteq\supp(Q_{G_2})$. Note that $U$ is an orthogonal transformation and hence, $UU^\top=I$ and $\det(U)=1 $ or $-1$.
\begin{itemize}
\item If $\det(U)=1$, it means that $U$ is a rotation and we are done by choosing $U^{(1)}=U$. 
\item If $\det(U)=-1$ (i.e., $U$ is an improper rotation), all we need is to find an orthogonal transformation $V$, such that (a) $\supp(Q_1U)=\supp(Q_1UV)$, i.e., it does not change the support, (b) $\det(V)=-1$, which implies that $\det(UV)=1$. That is, adding the transformation $V$ to $U$ does not change the support but makes the combination $UV$ into a rotation. Finding such a $V$ is easy, simply choosing a diagonal matrix with an odd number of diagonal entries equal to $-1$ and the rest equal to $1$. This will not change the support and only changes the sign of a subset of the entries. Therefore, we are done by choosing $U^{(1)}=UV$. Note that we are not forced to add a specific reflection at the end, we just add a particular one to do a sign flipping to show that the improper rotation can be changed into a rotation.

\end{itemize}
%$there exists another reflection $U'$, the one which fixes the zeros and changes the sign of a subset of the rest of the entries, which does not change the support of $Q_1U$, i.e., $\supp(Q_1U)=\supp(Q_1UU')$, and $UU'$ is a rotation. 

\section{Proof of Proposition \ref{prop:effect}}

\begin{itemize}
\item If $\xi_{i,j}=0$, then by definition, the Givens rotation corresponding to $A(i,j,k)$ is a zero degree rotation. Therefore, applying $A(i,j,k)$ has no effect.
\item If $\xi_{i,j}=\xi_{i,k}=\times$, then there exists a matrix $Q$ for which zeroing $\xi_{i,j}$ is an acute rotation and the other rows of $Q$ either have no element in the $(j,k)$ plane, or if they do, they will not become aligned with either $j$ or $k$ axis in the $(j,k)$ plane after the rotation. Therefore, support $(0,0)$ will stay at $(0,0)$, and any other support will become $(\times,\times)$.
\item If $\xi_{i,j}=\times$ and $\xi_{i,k}=0$, then the $i$-th row has been aligned with the $j$ axis in the $(j,k)$ plane before the rotation and since the rotation is planar, will become aligned with the $k$ axis after the rotation, and hence we have a $\pi/2$ rotation. Therefore, all other rows aligned with one axis will become aligned with the other axis, and any vector not aligned with either axes will remain the same. Therefore, we have support transformations $(\times,0)\rightarrow(0,\times)$, $(0,\times)\rightarrow(\times,0)$, $(\times,\times)\rightarrow(\times,\times)$, and $(0,0)\rightarrow(0,0)$, which is equivalent to switching columns $j$ and $k$.
\end{itemize}

\section{Proof of Theorem \ref{thm:mainrot2}}

We first prove the following weaker result:
\begin{theorem}
\label{thm:mainrotapp}
Let $\xi_1$ and $\xi_2$ be the support matrices of directed graphs $G_1$ and $G_2$, respectively.
$G_1$ is distribution equivalent to $G_2$ if and only if both following conditions hold:
\begin{itemize}
\item There exists a sequence of support rotations that maps $\xi_1$ to a subset of $\xi_2$.
\item There exists a sequence of support rotations that maps $\xi_2$ to a subset of $\xi_1$.
\end{itemize}
\end{theorem}
We need the following lemma for the proof.
\begin{lemma}
\label{lem:thm1}	
Consider a matrix $Q$ and a support matrix $\xi$.
If the support matrix of $Q$ is a subset of $\xi$, then for all $i$, $j$, $k$, the support matrix of $QG(j,k,\theta)$ is subset of $\xi A(i,j,k)$, where,
\[
\theta=
\begin{cases}
0,\quad  &\text{if } Q_{i,j}=Q_{i,k}=0 \text{ and } \xi_{i,j}=\xi_{i,k}\neq 0,\\
0,\quad  &\text{if } Q_{i,j}=Q_{i,k}=0 \text{ and } \xi_{i,k}\neq\xi_{i,j}= 0,\\
\pi/2,\quad  &\text{if } Q_{i,j}=Q_{i,k}=0 \text{ and } \xi_{i,j}\neq\xi_{i,k}= 0,\\
\tan^{-1}(-Q_{i,j}/Q_{i,k}),\quad  &\text{otherwise}.
\end{cases}
\]
\end{lemma}
\begin{proof}	
The rotation and the support rotation do not alter any columns except the $j$-th and $k$-th columns. Hence we only need to see if the desired property is satisfied by those two columns. If the support of $Q$ and $\xi$ are the same on those two columns, the desired result follows from the definition of support rotation. Otherwise, 
\begin{itemize}
\item If the support of $(Q_{i,j},Q_{i,k})$ is the same as $(\xi_{i,j},\xi_{i,k})$, then the effect of the rotation on $Q$ is the same as the effect of the support rotation on $\xi$, except that if we are in the second case of Proposition \ref{prop:effect}, the support rotation cannot introduce any extra zeros in rows $[p]\setminus\{i\}$, while this is possible for the rotation on $Q$. Therefore, the support matrix of $QG(j,k,\theta)$ is subset of $\xi A(i,j,k)$.

\item If $Q_{i,j}\neq0$ and $Q_{i,k}=0$, and $(\xi_{i,j},\xi_{i,k})=(\times,\times)$, then the rotation is a $\pm\pi/2$ while we have an acute rotation for $\xi$ (second case of Proposition \ref{prop:effect}). Hence, if a zero entry of $Q$ in a row in $[p]\setminus\{i\}$ has become non-zero after the rotation, $\xi$ has non-zero entries in both entries of that row. Therefore, the support matrix of $QG(j,k,\theta)$ is subset of $\xi A(i,j,k)$.

\item If $[Q_{i,j}=0$ and $Q_{i,k}\neq0$, and $(\xi_{i,j},\xi_{i,k})=(\times,\times)]$,
or $[Q_{i,j}=0$ and $Q_{i,k}=0$, and $(\xi_{i,j},\xi_{i,k})=(0,\times)]$,
or $[Q_{i,j}=0$ and $Q_{i,k}=0$, and $(\xi_{i,j},\xi_{i,k})=(\times,\times)]$, then the rotation has no effect on $Q$, while the support rotation can only turn some of the zero entries in rows $[p]\setminus\{i\}$ to non-zero. Therefore, the support matrix of $QG(j,k,\theta)$ is subset of $\xi A(i,j,k)$.

\item Finally, if $[Q_{i,j}=0$ and $Q_{i,k}=0$, and $(\xi_{i,j},\xi_{i,k})=(\times,0)]$, then by the statement of the lemma, the rotation on $Q$ will be $\pi/2$. Due to this fact and part three of Proposition \ref{prop:effect}, for both $Q$ and $\xi$, columns $j$ and $k$ will be flipped. Therefore, the support matrix of $QG(j,k,\theta)$ is subset of $\xi A(i,j,k)$.
\end{itemize}
\end{proof}

\begin{proof}[Proof of Theorem \ref{thm:mainrotapp}]
By Propositions \ref{prop:rot}, it suffices to show that there exists a sequence of support rotations $A_1,\cdots A_m$, such that $\xi_1 A_1,\cdots A_m\subseteq\xi_2$ if and only if for all choices of $Q_1$, there exists a sequence of Givens rotations $G_1,\cdots G_{m'}$ such that $\supp(Q_1G_1,\cdots G_{m'})\subseteq\supp(Q_{G_2})$.

Only if side:\\
For any matrix $Q_1$, by definition, the support matrix of $Q_1$ is a subset of $\xi_1$. In the sequence of support rotations, use the first support rotation $A_1(i,j,k)$ to generate Givens rotation $G_1(j,k,\theta)$, where $\theta$ is defined in the statement of Lemma \ref{lem:thm1}. Therefore, by Lemma \ref{lem:thm1}, the support matrix of $Q_1G_1(j,k,\theta)$ is a subset of $\xi_1A_1(i,j,k)$. Repeating this procedure, we see that the support matrix of $Q_1G_1,\cdots G_{m}$ is a subset of $\xi_1 A_1,\cdots A_m$. Now, by the assumption, $\xi_1 A_1,\cdots A_m\subseteq\xi_2$, and by definition, $\supp(\xi_2)=\supp(Q_{G_2})$. Therefore, $\supp(Q_1G_1,\cdots G_{m})\subseteq\supp(Q_{G_2})$.

If side:\\ 
Consider Givens rotation $G(j,k,\theta)$ applied to matrix $Q$. The effect of this rotation is one of the following:
\begin{enumerate}
	\item For an acute rotation, zeroing a subset of entries in columns $j$ and $k$.
	\item For a $\pm\pi/2$ rotation, swapping the support of columns $j$ and $k$.
	\item For an acute rotation, making no entries zero, while making a subset of the entries in columns $j$ and $k$ non-zero.
	\item For an acute rotation, no change to $\supp(Q)$.
\end{enumerate}
Since the assumption is true for all $Q$, we focus on matrices with support matrix $\xi_1$ (i.e., none of the free parameters are set at zero). If in case 1 above the subset has more than one element, more than one rows of $Q$ have been aligned on the $(j,k)$ plane, not on the $j$ and $k$ axes. Therefore, there exists another $Q$ (i.e., another choice of free parameters), in which those rows are not aligned. Consider $Q^*$ for which no such alignment happens, and hence, each of the Givens rotations in its sequence of rotations that causes case 1 above, only makes one entry zero. Therefore, its corresponding sequence of rotations acts exactly the same as support rotations for effects 1 and 2 above, in terms of their effect on the support. 

Hence, the proof is complete by showing that cases 3 and 4 can be ignored, because we assumed that the support matrix of $Q^*$ is $\xi_1$, and each not ignored Givens rotation corresponds to a support rotation, and by definition, $\supp(Q_{G_2})=\supp(\xi_2)$.
Clearly, case 4 can be ignored as it has no effect on the support. For case 3, we note that this effect only adds elements to the support, and hence we want the support after rotations to be a subset of $\supp(Q_{G_2})$, the rotations of this type do not serve for that purpose. Therefore, if we ignore such rotations, the resulting support would be smaller compared to the case of considering these rotations. Note that if due to such rotation entry $Q_{i,j}$ has become non-zero and later in the sequence there exists a type 1 rotation making $Q_{i,j}$ zero again, we already have zero in position $(i,j)$ and that type 1 rotation should be ignored as well.

\end{proof}

%%%%%%%%%%%%%%%%%%%%%%%%%%%%

Similar to the notion of distribution set, for a support matrix $\xi$ we define
\[
\Theta(\xi)\coloneqq\{\Theta:\Theta=\tilde{Q}\tilde{Q}^\top,\textit{ for any }\tilde{Q}\textit{ s.t. }\supp(\tilde{Q})\subseteq\supp(\xi)\}.
\]
Note that unlike $Q$, the matrix $\tilde{Q}$ is allowed to have zeros on its diagonal.
\begin{definition}
A support rotation mapping $\xi$ to $\xi'$ is lossless if $\Theta(\xi)=\Theta(\xi')$.
\end{definition}
Similar to the test for distribution equivalence, losslessness can be evaluated by checking if there exists a sequence of support rotations that maps $\xi'$ back to a subset of $\xi$. Clearly, reduction, reversible acute rotation, and column swap are lossless, as they are reversible. 
In most of the cases, irreversible acute rotations are lossy and lead to expansion of $\Theta(\xi)$, as it introduces capacity for having extra free variables. However, this is not necessarily the case. 

We have the following observations regarding checking for distribution equivalence.
\begin{lemma}
\label{lem:loss}
All the support rotations for checking the distribution equivalence of two directed graphs should be lossless.
\end{lemma}
We need the following lemma for the proof.
\begin{lemma}
\label{lem:subset}
If support matrix $\xi$ is mapped to $\xi'$ via a  support rotation, then $\Theta(\xi)\subseteq\Theta(\xi')$.
\end{lemma}
\begin{proof}	
For reduction, reversible acute rotation, and column swap, we have $\Theta(\xi)=\Theta(\xi')$, and irreversible acute rotation only introduces extra free variables, and hence, leads to $\Theta(\xi)\subseteq\Theta(\xi')$. To make the argument regarding irreversible acute rotation rigorous, consider irreversible acute rotation $A(i,j,k)$, which zeros $\xi_{i,j}$. For all $l\in[p]\setminus\{i\}$, if $\xi_{l,j}\neq\xi_{l,k}$, this rotation results in $(\xi_{l,j},\xi_{l,k})=(\times,\times)$. Suppose $(\xi_{i',j},\xi_{i',k})=(0,\times)$.  $A(i',j,k)$ will be a reversible acute rotation for $\xi'$ and leads to $\xi''$ such that $\xi\subsetneq\xi''$. Therefore, $\Theta(\xi)\subseteq\Theta(\xi'')=\Theta(\xi')$.

\end{proof}
\begin{proof}[Proof of Lemma \ref{lem:loss}]

If support matrix $\xi$ is mapped to $\xi'$ via a lossy support rotation, i.e., $\Theta(\xi)\neq\Theta(\xi')$ then by Lemma \ref{lem:subset}, we have $\Theta(\xi)\subsetneq\Theta(\xi')$.
Suppose we want to check the equivalence of directed graphs $G_1$ and $G_2$ with support matrices $\xi_1$ and $\xi_2$, respectively. We note that $\Theta(G_1)=\Theta(\xi_1)$. Suppose $\xi_1$ is mapped to $\xi$ through a sequence of
support rotations, including a lossy rotation, which in turn is mapped to $\xi'\subseteq\xi_2$. Therefore, 
\[
\Theta(G_1)=\Theta(\xi_1)\subsetneq\Theta(\xi)\subseteq\Theta(\xi')\subseteq\Theta(\xi_2)=\Theta(G_2).
\]
Therefore,
\[
\Theta(G_1)\neq\Theta(G_2).
\]
\end{proof}

Using Lemma \ref{lem:loss}, we can prove Theorem \ref{thm:mainrot2}:

\begin{proof}
The if side is clear by Theorem \ref{thm:mainrotapp}. For the only if side, by Theorem \ref{thm:mainrotapp} and Lemma \ref{lem:loss} we show that if $\xi_1$ can be mapped to $\xi_2$ via a sequence of lossless support rotations (i.e., $\Theta(\xi_1)=\Theta(\xi_2)$) including an irreversible acute rotation, then there exists a sequence of support rotations which does not include any irreversible acute rotations that maps $\xi_1$ to a subset of $\xi_2$.

We show that every irreversible acute rotation can be replaced by other types of support rotation. Consider the first irreversible acute rotation $A(i,j,k)$ in the sequence, which maps $\xi$ to $\xi'$. 
Applying this rotation, we have $(\xi'_{i,j},\xi'_{i,k})=(0,\times)$, and columns $\xi'_{\cdot,j}$ and $\xi'_{\cdot,k}$ agree on the rest of the entries.
Suppose, prior to applying this rotation, columns $\xi_{\cdot,j}$ and $\xi_{\cdot,k}$ disagree on $m$ entries in rows with indices $\textit{diff}=\{s_1,\cdots,s_m\}$. Let 
\[
\textit{diff}_j=\{l:l\in\textit{diff},\xi_{l,j}=0\},
\]  
\[
\textit{diff}_k=\{l:l\in\textit{diff},\xi_{l,k}=0\}, 
\]
and 
\[
M=
\begin{cases}
\max\{m_j,m_k\},&\quad m_j\neq m_k,\\
m_j+1,&\quad\textit{otherwise}.
\end{cases}
\]
where $m_j=|\textit{diff}_j|$ and $m_k=|\textit{diff}_k|$. We can always swap two columns, hence, without loss of generality, assume $M=m_j+\mathds{1}_{\{m_j=m_k\}}$.
\begin{claim}
\label{clm:1}

$\xi$ can be transformed via reduction and reversible acute rotation to a support matrix, in which there exist columns with indices $\{t_1,\cdots,t_{M-1}\}$ such that the sub-matrix of $\xi$ on columns $\{t_1,\cdots,t_{M-1},j,k\}$ and rows $\textit{diff}\cup\{i\}$ has a column with $i$ zeros, for all $i\in\{0,1,...,M\}$, and the sub-matrix of $\xi$ on columns $\{t_1,\cdots,t_{M-1},j,k\}$ and the rest of the rows has equal columns.
%Since $A(i,j,k)$ is lossless, there exist  columns with indices $\{t_1,\cdots,t_{M-1}\}$ such that the sub-matrix of $\xi$ on columns $\{t_1,\cdots,t_{M-1},j,k\}$ and rows $\textit{diff}\cup\{i\}$ can be transformed via reduction and reversible acute rotation to a support matrix which has the property that it has a column with $i$ zeros, for all $i\in\{0,1,...,M\}$.
\end{claim}
\begin{proof}[Proof of Claim \ref{clm:1}]
Since $A(i,j,k)$ is lossless, we can map $\xi'$ to a subset of $\xi$. Therefore, we should be able to introduce zeros in $\xi'$ in indices $\textit{diff}_j$ of column $j$ and indices $\textit{diff}_k$ of column $k$, without removing the existing zeros, except potentially $\xi'_{ij}$.
We first use a reversible acute rotation on columns $j$ and $k$ to move the newly introduce zero in $\xi'_{ij}$ to the first index in $\textit{diff}_j$, and we denote the resulting support matrix by $\xi^{(1)}$.
We note that reduction is the only support rotation, which increases the number of zeros in the support matrix. Therefore, we need one reduction for reviving each of the $m-1$ other removed zeros in the transformation of $\xi$ to $\xi'$.

The claim can be proven by induction. The base of the induction, i.e., for $M=2$ can be proven as follows:
\begin{itemize}
\item {\bf Case 1: }$m_j=m_k=1$. In order to have the zero in column $k$, we need to perform a reduction, for which, we need another column $\xi^{(1)}_{\cdot,t_1}$ equal to $\xi^{(1)}_{\cdot,k}$, i.e., $d_H(\xi^{(1)}_{\cdot,t_1},\xi^{(1)}_{\cdot,k})=0$, where $d_H(\cdot,\cdot)$ denotes the Hamming distance between its two arguments. Since the original irreversible acute rotation was on the $(j,k)$ plane and did not affect other columns, the column $t_1$ with the aforementioned property exists in the original support matrix $\xi$ as well, i.e., $\xi_{\cdot,t_1}=\xi^{(1)}_{\cdot,t_1}$. Now, a reversible acute rotation can be performed on columns $t_1$ and $k$ to set $d_H(\xi_{\cdot,j},\xi_{\cdot,j})=0$, and then a reduction can be performed to introduce another zero in column $j$ of $\xi$. The resulting support matrix has the desired property stated in the claim.
\item {\bf Case 2: }$m_j=2,m_k=0$. In order to have the zero in the second index of $\textit{diff}_j$, we need to perform a reduction, for which, we need another column equal to $\xi^{(1)}_{\cdot,j}$. This can be obtained by one of the following cases:
\begin{itemize}
\item There already exists a column $t_1$, such that $d_H(\xi^{(1)}_{\cdot,t_1},\xi^{(1)}_{\cdot,j})=0$. Similar to Case 1, This implies that column $t_1$ also exists in $\xi$. Therefore, $\xi$ has the desired property.
\item There exists a column $t_1$, such that $d_H(\xi^{(1)}_{\cdot,t_1},\xi^{(1)}_{\cdot,j})\neq0$, but $d_H(\xi^{(1)}_{\cdot,t_1},\xi^{(1)}_{\cdot,k})=1$. Similar to Case 1, This implies that column $t_1$ also exists in $\xi$. Therefore, a reversible acute rotation can transform $\xi$ to a support matrix with the desired property.
\item There exists a column $t_1$, such that  $d_H(\xi^{(1)}_{\cdot,t_1},\xi^{(1)}_{\cdot,k})=0$. Similar to Case 1, This implies that column $t_1$ also exists in $\xi$. Therefore, two reductions, one on columns $(t_1,k)$, and then one on columns $(t_1,j)$ can transform $\xi$ to a support matrix with the desired property.
\end{itemize}
\item {\bf Case 3: }$m_j=2,m_k=1$. In order to have the zero in column $k$, we need to perform a reduction, for which, we need another column $t_1$ equal to column $k$, i.e., $d_H(\xi^{(1)}_{\cdot,t_1},\xi^{(1)}_{\cdot,k})=0$. Similar to Case 1, This implies that column $t_1$ also exists in $\xi$. Therefore, $\xi$ has the property desired in the claim.
\end{itemize}
Now, suppose the property holds for $M=n$. To show that it also holds for $M=n+1$, a reasoning same as the one provided for the base case of the induction can be used, and it can be shown that for the required extra reduction, an extra column $t_n$ should exist in $\xi$.

\end{proof}

By Claim \ref{clm:1}, $\xi$ can be transformed via reduction and reversible acute rotation to a support matrix with the stated property. Therefore, we assume $\xi$ has the property. Therefore, we have columns $\{t_1,\cdots,t_{M-1},j,k\}$ with any number of zeros $0\le i\le M$ on rows $\textit{diff}\cup\{i\}$, and it is easy to see the $i$ zeros in these columns can be relocated to any other indices via only reversible acute rotations amongst these columns. Therefore, any effect sought to be achieved via columns $j$ and $k$ of $\xi'$, can be obtained via columns $\{t_1,\cdots,t_{M-1},j,k\}$ of $\xi$, and hence, the irreversible acute rotation could have been replaced by other types of rotations.

\end{proof}

\section{Proof of Proposition \ref{prop:dir}}

To show that the property holds for cycle $C=(X_1,\cdots,X_m,X_1)$, we note that our desired support matrix is $\xi_1$, when columns $2$ to $m$ are all shifted to left by one, and column $1$ is moved to location $m$. Therefore, it suffices to first flip columns $1$ and $2$, then $2$ and $3$, all the way to $m-1$ and $m$. For each flip, we use the third part of Proposition \ref{prop:effect}. For instance, for flipping columns $j$ and $j+1$, we find row $i$ such that $\xi_{i,j}\neq\xi_{i,j+1}$ (if there is no such row, then no flip for those columns is needed as they are already the same). If, say $\xi_{i,j}=\times$, we use support rotation $A(i,j,j+1)$ for flipping columns $j$ and $j+1$. Following the same reasoning, we see that  support rotation of $\xi_2$ leads to a subset of $\xi_1$.

\section{Proof of Proposition \ref{prop:col}}

If side:\\
If columns of $\xi_2$ are permutation of columns of $\xi_1$, then $\xi_1$ can be mapped to $\xi_2$ and vice versa via a sequence of column swap rotations. Therefore, by Theorem \ref{thm:mainrot2}, $G_1\equiv G_2$.

Only if side:\\
If $G_1\equiv G_2$, the by Theorem \ref{thm:mainrot2}, $\xi_1$ can be mapped to a subset of $\xi_2$ and $\xi_2$ can be mapped to a subset of $\xi_1$, both via only reductions, reversible acute rotations and column swaps. If each pair of column of $\xi_1$ are different in more than one entry, then we are not able to perform any reversible acute rotations and reductions. Therefore, we have been able to perform the mapping merely via column swaps. Therefore, columns of $\xi_2$ are permutation of columns of $\xi_1$.

\section{Proof of Proposition \ref{prop:red}}

%~\\
Only if side:\\
By definition, directed graph $G$ is reducible if there exists directed graph $G'$ such that $G\equiv G'$ and $\xi'\subset \xi$. By Theorem \ref{thm:mainrot2}, $\xi$ can be mapped to a subset of $\xi'$ via a sequence of support rotations comprised of reductions, reversible acute rotations and column swaps. We note that reduction is the only support rotation, which increases the number of zeros in the support matrix. 
Therefore, there should be a reduction in the sequence. We can always swap any two columns and the location of two columns does not influence the feasibility of reduction or reversible acute rotations. Therefore, column swaps can be ignored in reducibility.

If side:\\
Suppose the performed reduction turns a non-zero entry in column $j$ to zero, using a reduction on columns $j$ and $k$. Note that prior to the reduction, these columns have the same number of zeros and in order to be able to perform the reduction a sequence of reversible acute rotations have been performed to prepare column $k$ such that the hamming distance of columns $j$ and $k$ be equal to zero. That is, its zeros have been moved to match the zero pattern of column $j$. We can always assume that we only moved the zeros of column $k$, as if there are columns to move the zeros of column $j$, they can be used to move the zeros of column $k$ as well.
The only concern is that the zeroed entry may be on the diagonal. In this case, a reversible acute rotation can be performed on columns $j$ and $k$ to move the new zero to another index of column $j$.
Also, entry $(j,j)$ cannot be the only non-zero entry of column $j$; otherwise, column $k$ should also have only one non-zero entry, which should initially be located at $(k,k)$. Therefore, to perform a reversible acute rotation on any other column $l$ and $k$, column $l$ should have only two non-zero entries, on $(k,l)$ and $(j,l)$, while one of them should initially be located at $(l,l)$. This reasoning can be repeated $p$ times and leads to the contradiction that the final column is not allowed to have a non-zero entry on the diagonal, which contradicts the fact that $\xi$ is the support matrix corresponding to a directed graph.
 Finally, all the performed reversible acute rotations can be done in the reverse direction to obtain the initial zero pattern for columns $[p]\setminus\{j\}$.

\section{Proof of Proposition \ref{prop:2cycle}}

Using Proposition \ref{prop:red}, we show that for directed graph $G$ with support matrix $\xi$, if there exists a sequence of reversible support rotations that enables us to apply a reduction to $\xi$, then $G$ has a 2-cycle.
Suppose the reduction is performed on columns $j$ and $k$, to turn a non-zero entry of column $j$ to zero. 
If no reversible support rotations prior to the reduction is needed, it implies that already columns $j$ and $k$ are identical. Therefore, $\xi_{j,k}=\xi_{j,j}=\times$, and $\xi_{k,j}=\xi_{k,k}=\times$. Therefore, there exists a 2-cycle between $j$ and $k$ and the proof is complete. Therefore, we assume some reversible support rotations are needed.

Consider the first rotation in the sequence of reversible support rotations applied to column $k$. Assume it is performed on columns $t_1$ and $k$. Therefore, the support of column $t_1$ has one element more than the support of column $k$, and the Hamming distance between these two columns is one. The only way that this does not cause a 2-cycle between $t_1$ and $k$ is that $\xi_{t_1,k}=0$, and $\xi_{k,t_1}=\times$, and all the entries show be the same.
This rotation is supposed to move the extra zero in column $k$ to an index, which is zero in column $j$ (to reduce the Hamming distance between columns $j$ and $k$). Therefore, since after this rotation, $\xi_{t_1,k}$ will become non-zero, we should have $\xi_{t_1,j}=\times$. This will lead to a 2-cycle unless if $\xi_{j,t_1}=0$. Now, if $\xi_{j,t_1}=0$, because all the entries of columns $t_1$ and $k$ where the same, we also have $\xi_{j,k}=0$. This gives us two options for $\xi_{k,j}$:
\begin{itemize}
\item	If $\xi_{k,j}=0$, then we need another column $t_2$ so that we perform a reversible acute rotation on columns $t_2$ and $k$ to move $\xi_{j,k}=0$ to entry $\xi_{k,k}$, which is currently non-zero. This means that columns $t_2$ and $k$ should be the same on all the entries, except that $\xi_{j,t_2}=\times$, but $\xi_{j,k}=0$. Therefore, $\xi_{k,t_2}=\xi_{k,k}=\times$ and $\xi_{t_2,k}=\xi_{t_2,t_2}=\times$, which implies that there is a 2-cycle between $t_2$ and $k$.
\item	If $\xi_{k,j}=\times$, then in order for columns $k$ and $j$ to have the same number of non-zero entries, there should exist index $l$ such that $\xi_{l,k}=\times$, and $\xi_{l,j}=0$. Now, we need another column $t_2$ so that we perform a reversible acute rotation on columns $t_2$ and $k$ to move $\xi_{j,k}=0$ to entry $\xi_{l,k}$. This means that columns $t_2$ and $k$ should be the same on all the entries, except that $\xi_{j,t_2}=\times$, but $\xi_{j,k}=0$. Therefore, $\xi_{k,t_2}=\xi_{k,k}=\times$ and $\xi_{t_2,k}=\xi_{t_2,t_2}=\times$, which implies that there is a 2-cycle between $t_2$ and $k$.
\end{itemize}

\section{Proof of Corollary \ref{cor:DAG}}

We first prove the following corollary:
\begin{corollary}
\label{cor:irred}
Irreducible directed graphs $G_1$ and $G_2$ with support matrices $\xi_1$ and $\xi_2$ are equivalent if and only if there exist sequences of reversible acute rotations and column swaps that map their support matrices to one another.
\end{corollary}
\begin{proof}
By Proposition \ref{prop:red}, there exists no sequence of reversible acute rotations that enables us to apply a reduction to the support matrix. Therefore, we only need to consider reversible acute rotations and column swaps, and we need to map one support matrix to the other, rather than mapping it to a subset of the other.

\end{proof}

\begin{proof}[Proof of Corollary \ref{cor:DAG}]
DAGs do not have 2-cycles. Therefore, by Proposition \ref{prop:2cycle}, DAGs are irreducible. Therefore, the result follows from Corollary \ref{cor:irred}.

\end{proof}

% \section{Proof of the First Part of Example \ref{ex:chain}}

% To see that $G_1:X_1\rightarrow X_2\rightarrow X_3\equiv G_2:X_1\leftarrow X_2\leftarrow X_3, $, we note that
% \[
% \xi_1=
% \begin{bmatrix}
%     \times & \times & 0 \\
%     0 & \times & \times \\
%     0 & 0 & \times
% \end{bmatrix}
% \underrightarrow{~~A(1,2,1)~~}
% \begin{bmatrix}
%     \times & 0 & 0 \\
%     \times & \times & \times \\
%     0 & 0 & \times
% \end{bmatrix}
% \underrightarrow{~~A(2,3,2)~~}
% \begin{bmatrix}
%     \times & 0 & 0 \\
%     \times & \times & 0 \\
%     0 & \times & \times
% \end{bmatrix}
% \subseteq\xi_2.
% \]
% \[
% \xi_2=
% \begin{bmatrix}
%     \times & 0 & 0 \\
%     \times & \times & 0 \\
%     0 & \times & \times
% \end{bmatrix}
% \underrightarrow{~~A(3,2,3)~~}
% \begin{bmatrix}
%     \times & 0 & 0 \\
%     \times & \times & \times \\
%     0 & 0 & \times
% \end{bmatrix}
% \underrightarrow{~~A(2,1,2)~~}
% \begin{bmatrix}
%     \times & \times & 0 \\
%     0 & \times & 0 \\
%     0 & \times & \times
% \end{bmatrix}
% \subseteq\xi_1.
% \]

\section{Proof of Theorem \ref{thm:graphical}}

If side:\\
If there exist sequences of parent reduction, parent exchange, and cycle reversion, mapping one graph to a subgraph of the other, then there exist sequences of reduction, reversible acute rotation, and column swap mapping the support matrix of one graph to a subset of the support matrix of the other. Therefore, by Theorem \ref{thm:mainrot2}, $G_1$ is distribution equivalent to $G_2$.

Only if side:\\
The proof of the only if side consists of two steps:
\begin{itemize}
\item {\bf Step 1. }We note that
\begin{enumerate}
\item All support rotations of reduction type, that do not make a diagonal entry zero are representable by a parent reduction. This is clear from the definitions of reduction and parent reduction.
\item All reversible acute rotations, that do not make a diagonal entry zero are representable by a parent exchange. This is clear from the definitions of reversible acute rotation and parent exchange.
\item If we have a reversible acute rotation and a column swap on columns $j$ and $k$ such that the reversible acute rotation makes the diagonal entry $\xi_{j,j}$ zero and then the column swap swaps columns $j$ and $k$ (we call such a pair a flip pair), then this pair can be replaced by a reversible acute rotation that makes the non-diagonal entry $\xi_{j,k}$ zero, and hence, is representable by a parent exchange.
\item If we start with a support matrix with no diagonal entries equal to zero and by performing a sequence of column swaps reach another support matrix with no diagonal entries equal to zero, then this sequence is representable by a cycle reversion. To see this, we note that if after the sequence of column swaps, column $j$ has moved to location $k$, it implies that its $j$-th and $k$-th elements are non-zero. Therefore, the original support matrix corresponds to a graph containing the edge $j\rightarrow k$, and the final support matrix corresponds to a graph containing the edge $k\rightarrow j$. This reasoning identifies the cycle before, and the reversed cycle after the transformation.
\end{enumerate}
\end{itemize}
Step 1 implies that if we have a sequence of support rotations which includes
1. reduction rotations, that do not make a diagonal entry zero,
2. reversible acute rotations, that do not make a diagonal entry zero,
3. flip pairs, and 
4. sequence of column swaps starting and ending on a support matrix with non-zero diagonal entries,
(we call such a sequence, a representable sequence) then we can represent this sequence with a sequence of parent reductions, parent exchanges, and cycle reversions.
\begin{itemize}
\item {\bf Step 2. }If $G_1$ is distribution equivalent to $G_2$, then by Theorem \ref{thm:mainrot2}, there exists a sequence of reduction, reversible acute rotations, and column swap mapping the support matrix of one to the other. We show that in this case, there exists a representable sequence as well that maps the support matrix of one to the other. Therefore, by Step 1 the only if side will be concluded.

We note that since $\xi_1$ is a support matrix of a directed graphs, it does not have any zeros on the main diagonal.
Given the sequence of support rotations, the column swaps do not enable us or prevent us from performing reversible acute rotations and reductions,  and merely change the indices of the columns. Therefore, we can have an equivalent sequence of support rotations, in which we have moved all the column swaps, except those involved in flip pairs, to the end of the sequence. 
Consider the first rotation in the sequence of the rotations which zeros out a diagonal entry. If this rotation is of reduction type and has zeroed out $\xi_{i,i}$ using columns $i$ and $j$, then $\xi_{i,j}$ should have been non-zero. Therefore, we can instead replace it by zeroing $\xi_{i,j}$, and use column $j$ instead of column $i$ in the next steps. If this rotation is of reversible acute rotation type and has zeroed out $\xi_{i,i}$ using columns $i$ and $j$, then $\xi_{i,j}$ should have been non-zero. Therefore, again we can instead replace it by zeroing $\xi_{i,j}$, and use column $j$ instead of column $i$ in the next steps. Therefore, we can perform all the reductions and reversible acute rotations and from $\xi_1$ obtain $\xi'_1$, which does not have any zeros on the main diagonal, and via a sequence of column swaps can be mapped to a subset of $\xi_2$.

Now, we perform the reverse of that sequence of column swaps on $\xi_2$, which gives us a superset of $\xi'_1$ (call it $\xi''_2$), and hence, does not have any zeros on the main diagonal. Therefore, since $\xi_2$ is a support matrix of a directed graph and hence, it also does not have any zeros on the main diagonal, by part 4 of Step 1, this is equivalent to a cycle reversion. $\xi''_2$ is a superset of $\xi'_1$, and both $\xi''_2$ and $\xi'_1$ are graphically representable. By Lemma \ref{lem:loss}, the corresponding directed graph of $\xi''_2$ is the same (if the directed graph corresponding to $\xi''_2$ is irreducible) or reducible to the directed graph corresponding to $\xi'_1$. Therefore, by Proposition \ref{prop:red} we can perform the reduction via a sequence of reversible acute rotations. Similar to the reasoning in the previous paragraph, since we start with a support matrix with no zeros on the main diagonal, this can be done without zeroing any element of the main diagonal, and hence, we can map $\xi''_2$ to $\xi'_1$. Finally, reversing the reversible acute rotations of the sequence from $\xi_1$ to $\xi'_1$, we obtain a subset of $\xi_1$, and the whole sequence from $\xi_2$ to a subset of $\xi_1$ is a representable sequence. Similarly, we can construct a representable sequence mapping $\xi_1$ to a subset of $\xi_2$, which completes the proof.

\end{itemize}

\section{Proof of Corollary \ref{cor:DAGgraphical}}

DAGs do not have 2-cycles. Therefore, by Proposition \ref{prop:2cycle}, DAGs are irreducible. Hence, a parent reduction cannot be performed. Also, DAGs do not have cycles. Hence, there will not be any cycle reversions. Therefore, the result follows from Theorem \ref{thm:graphical}.

\section{Proof of Proposition \ref{prop:mes0}}
To violate faithfulness, there are finite number of sets of hard constraints that should be satisfied (since hard constraints are distributional constraints and hence limited).
Let $\theta_i$ be the set of values satisfying the $i$-th set of constraints. By the definitions of hard constraints, $\theta_i$ is Lebesgue measure zero. Therefore, the set of distributions not g-faithful to $G$, which is the finite union is also Lebesgue measure zero.

\section{Proof of Proposition \ref{prop:consistent}}

Suppose $G^*$ is the ground truth DG and it generates distribution $\Theta$, and $G_1$ is a candidate DG which we want to decide whether it is the ground truth or not.

Suppose $G_1\cong G^*$. Then there exists a set of distribution with non-zero Lebesgue measure that both $G_1$ and $G^*$ can generate. Suppose $\Theta$ is a distribution coming from this intersection which also satisfies Assumption \ref{ass:G_faith}. Then clearly, since both DGs can generate $\Theta$, there is no way to realize which one has been the ground truth, and hence, $G_1$ is non-identifiable from $G^*$.

For the opposite direction, suppose $G_1\not\cong G^*$ then either there is no distribution that they can both generate, or the measure of such distributions is zero. In the first case, $\Theta$ is not generatable by $G_1$ and hence we can identify that $G_1$ is not the ground truth. In the second case, by Assumption \ref{ass:G_faith}, $\Theta$ cannot be from the intersection and hence again is not generatable by $G_1$ and hence we can identify that $G_1$ is not the ground truth.

\section{Proof of Theorem \ref{thm:l0}}

Let $G^*$ and $\Theta$ be the ground truth structure and the generated distribution, and for an ML estimator, assume we are capable of finding a correct pair $(\hat{B}_{ML},\hat{\Omega}_{ML})$, such that $(I-\hat{B}_{ML})\hat{\Omega}_{ML}^{-1}(I-\hat{B}_{ML})^{\top}=\Theta$ and denote the directed graph corresponding to $\hat{B}_{ML}$ by $\hat{G}_{ML}$. We have $\Theta\in\Theta(\hat{G}_{ML})$, which implies that $\Theta$ contains all the distributional constraints of $\hat{G}_{ML}$. Therefore, under Assumption \ref{ass:G_faith}, we have 
%the property that the hard distributional constraints of $\hat{G}_{ML}$ is a subset of the hard distributional constraints of $G^*$.
$H(\hat{G}_{ML})\subseteq H(G^*)$.

Let $(\hat{B}_{\ell_0},\hat{\Omega}_{\ell_0})$ be the output of $\ell_0$-regularized ML estimator, and denote the directed graph corresponding to $\hat{B}_{\ell_0}$ by $\hat{G}_{\ell_0}$.
Since the likelihood term increases much faster with the sample size compared to the penalty term, asymptotically, we still have the desired properties that $\Theta$ contains all the distributional constraints of $\hat{G}_{\ell_0}$, and hence, under Assumption \ref{ass:G_faith}, we again have 
$H(\hat{G}_{\ell_0})\subseteq H(G^*)$.

Now, consider an irreducible equivalent of $G^*$, denoted by $G^\dagger$. Since $H(G^*)=H(G^\dagger)$, we have 
$H(\hat{G}_{\ell_0})\subseteq H(G^\dagger)$. Also, because of the penalty term we have $|E(\hat{G}_{\ell_0})|\le|E(G^\dagger)|$, otherwise the algorithm would have outputted $G^\dagger$. Therefore, by Assumption \ref{ass:G_faith}, we have $H(\hat{G}_{\ell_0})=H(G^\dagger)$, and hence $H(\hat{G}_{\ell_0})=H(G^*)$. Therefore, by definition, $\hat{G}_{\ell_0}\cong G^*$.

%\newpage

\section{Algorithm for Enumerating Members of a Distribution Equivalence Class and Determining the Equivalence of Two Structures}

We first propose an algorithm for enumerating members of the distribution equivalence class of a directed graph with support matrix $\xi$, based on a depth-first traversal. The algorithm is based on a search tree that is rooted at $\xi$ and branches out via $\textsc{Reduction}$ and $\textsc{AcuteRotation}$ operations.
These two operations
%when admissible for input support matrix $\xi$, 
are defined in Algorithm \ref{algorithm:operations}.
Since those two rotation operations are independent of column swaps, we perform a similar depth-first traversal of column swaps at the end, leveraging the graphical, cycle reversion representation for efficiency. 

\begin{algorithm}[H]
\footnotesize
\begin{algorithmic}[1]
    \Function{$\textsc{Reduction}$}{$\xi,i,j$}
        \State Initialize $\xi'\gets \xi$
        \State $\xi'_{i,j} \gets 0$
        \State \textbf{return} $\xi'$
    \EndFunction
    \State
    \Function{$\textsc{AcuteRotation}$}{$\xi,i,j,k,\ell$}
        \State Initialize $\xi'\gets \xi$
        \State $\xi'_{i,j} \gets 0$
        \State $\xi'_{\ell,j}\gets 1$
        \State $\xi'_{\ell,k}\gets 1$
        \State \textbf{return} $\xi'$
    \EndFunction
\end{algorithmic}
\caption{Reduction and Acute Rotation Operations}
\label{algorithm:operations}
\end{algorithm}

Each vertex in the search tree corresponds to a support matrix and each of its children corresponds to the outputs of an admissible $\textsc{Reduction}$ and $\textsc{AcuteRotation}$ operation. 
Algorithm \ref{algorithm:find_rotations} represents the pseudo-code of the function which compiles a set of those operations for a given support matrix.

\begin{algorithm}[H]
\footnotesize
\begin{algorithmic}[1]
    \Function{FindRotations}{$\xi$}
    \State Initialize $Rotations = \emptyset$
    \State \textit{// Find Legal Reductions}
    \For{$j,k$ such that $\|\xi_{\cdot,j}-\xi_{\cdot,k}\|_1 = 0$}
        \For {$i$ such that $\xi_{i,j} = 1$}
            \If {$i \neq j$}
                \State $Rotations\gets Rotations\cup \{\textsc{Reduction}(\xi,i,j)\}$
            \EndIf
            \If {$i \neq k$}
                \State $Rotations\gets Rotations\cup \{\textsc{Reduction}(\xi,i,k)\}$
            \EndIf
        \EndFor
    \EndFor
    \State \textit{// Find Legal Acute Rotations}
    \For{$j,k$ such that $\|\xi_{\cdot,j}-\xi_{\cdot,k}\|_1 = 1$}
        \State $\ell \gets$ index such that $\xi_{\ell,j}\ne\xi_{\ell,k}$
        \For {$i\ne \ell$ such that $\xi_{i,j}=1$}
            \If{$i\ne j$}
                \State $Rotations\gets Rotations\cup \{\textsc{AcuteRotation}(\xi,i,j,k,\ell)\}$
            \EndIf
            \If{$i\ne k$}
                \State $Rotations\gets Rotations\cup \{\textsc{AcuteRotation}(\xi,i,k,j,\ell)\}$
            \EndIf
        \EndFor
    \EndFor
    \State \textbf{return} $Rotations$
    \EndFunction
\end{algorithmic}
\caption{Finding Legal Rotations}
\label{algorithm:find_rotations}
\end{algorithm}

Algorithm \ref{algorithm:enum} enumerates the equivalence class. The algorithm keeps track of the search tree state using a stack $S$ which contain sets of rotated support matrices. The first step of the algorithm enumerates a subset of the equivalence class of $\xi^*$ by finding sequences of $\textsc{Reduction}$ and $\textsc{AcuteRotation}$ operations. The second step enumerates column swaps in a similar depth-first fashion. It is made efficient by using the fact that sequences of legal column swaps correspond to sequences of cycle reversions.

\begin{algorithm}[H]
\footnotesize
\begin{algorithmic}[1]
    \Function{ReverseCycles}{$\xi$}
        \State $Reversed\gets \emptyset$
        \State $\mathcal{C}\gets$ list of cycles in $\xi$
        \For{$C$ in $\mathcal{C}$}
            \State $\xi'\gets$ Column-permuted $\xi$ with cycle $C$ reversed
            \State $Reversed\gets Reversed \cup\{\xi'\}$
        \EndFor
        \State \textbf{return} $Reversed$
    \EndFunction
    \State {}
    \Procedure{EnumerateEquiv}{$p\times p$ support matrix $\xi^*$}
        \State Initialize $Equiv\gets \{\xi^*\}$.
        \State Initialize empty stack $S$
        \State $S.push(\textsc{FindRotations}(\xi^*))$
        \While{$S$ is not empty}
            \State $Rotations\gets S.pop()$
            \If{$|Rotations| = 0$}
                \State \textbf{continue}
            \Else
                \State $\xi \gets$ a support matrix in the set $Rotations$
                \State $Rotations\gets Rotations\setminus\{\xi\}$
                \State $S.push(Rotations)$
                \If{$\xi$ not in $Equiv$}
                    \State $Equiv\gets Equiv \cup \{\xi\}$
                    \State $S.push(\textsc{FindRotations}(\xi))$
                \EndIf
            \EndIf
        \EndWhile
        \State \textit{// Enumerate legal column swaps via cycle reversion}
        \For{$\tilde{\xi}$ in $Equiv$}
            \State Initialize empty stack $S$
            \State $S.push(\textsc{ReverseCycles}(\tilde{\xi}))$
            \While {$S$ is not empty}
                \State $Reversals\gets S.pop()$
                \If {$|Reversals| = 0$}
                    \State \textbf{continue}
                \Else
                    \State $\xi\gets$ a support matrix in the set $Reversals$
                    \State $Reversals\gets Reversals\setminus\{\xi\}$
                    \State $S.push(Reversals)$
                    \If{$\xi$ not in $Equiv$}
                        \State $Equiv\gets Equiv\cup \{\xi\}$
                        \State $S.push(ReverseCycles(\xi))$
                    \EndIf
                \EndIf
            \EndWhile
            % \State Initialize $C_{seen} \gets \emptyset$
            % \State $\xi\gets\tilde{\xi}$ 
            % \State $C \gets$ Find a cycle $C$ in $\xi$ that is previously unseen (not in $C_{seen}$)
            % \While{$C$ exists}
            %     \State $\xi' \gets$ Column-permuted $\xi$ with cycle $C$ reversed
            %     \If{$\xi'$ not in $Equiv$}
            %         \State $Equiv \gets Equiv \cup \{\xi'\}$
            %         \State $\xi\gets\xi'$
            %     \EndIf
            % \State $C_{seen}\gets C_{seen}\cup\{C\}$
            % \State $C \gets$ Find a cycle $C$ in $\xi$ that is previously unseen (not in $C_{seen}$)
            % \EndWhile
        \EndFor
    \EndProcedure
\end{algorithmic}
\caption{Enumerating equivalent structures}
\label{algorithm:enum}
\end{algorithm}

Finally, the procedure $\textsc{EnumerateEquiv}$ in Algorithm \ref{algorithm:enum} may be used to determine whether or not two DGs with respective support matrices $\xi_1$ and $\xi_2$ are equivalent by enumerating the equivalence class of $\xi_1$ and checking whether or not $\xi_2$ is in that equivalence class.

\section{Virtual Edge Search Operator}

\begin{figure}[t]
\begin{center}
\includegraphics[scale=0.43]{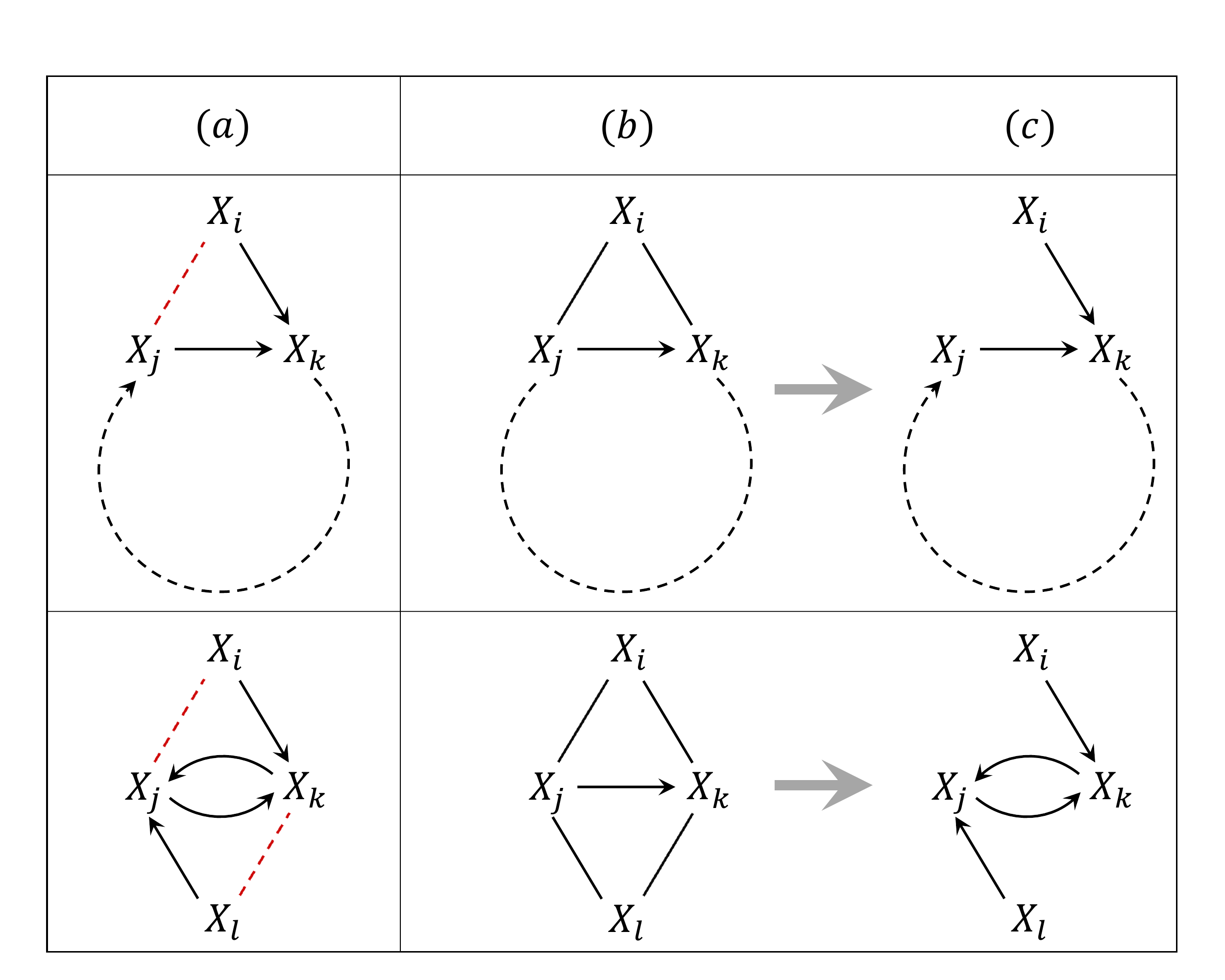}	
\end{center}
\caption{Virtual edge search operator.}
\label{fig:vedge}
\end{figure}

For acyclic DGs, under the Markov and faithfulness assumptions, a variable $X_i$ is adjacent to a variable $X_j$ if and only if $X_i$ and $X_j$ are dependent conditioned on any subset of the rest of the variables. This is not the case for cyclic DGs \citep{richardson1996polynomial}. Two non-adjacent variables $X_i$ and $X_j$ are dependent conditioned on any subset of the rest of the variables if they have a common child $X_k$ which is an ancestor of $X_i$ or $X_j$. In this case, we say there exists a virtual edge between $X_i$ and $X_j$. Figure \ref{fig:vedge}(a) demonstrates two examples. In this figure, virtual edges are shown with dashed red edges.

There are two cases that detecting a virtual edge as a real edge can trap the greedy search into a local optima which can be improved.\\

\noindent
{\bf Case 1.} This case is shown in the first row of Figure \ref{fig:vedge}.
If a greedy search algorithm finds the edges between $X_k$ and $X_j$ but does not find $X_k$ and $X_j$ to be on a cycle, that is, if it does not find the directions correctly, it can significantly increase the likelihood by adding an edge at the location of the virtual edge between $X_i$ and $X_j$. The algorithm would therefore be trapped in a local optimum shown in Figure \ref{fig:vedge}(b) with one more edge than the ground truth shown in Figure \ref{fig:vedge}(c).
To resolve this issue, we propose adding the following search operator: 
Suppose we have a triangle over three variables $X_i$, $X_j$ and $X_k$, and there exists an additional sequence of edges connecting $X_j$ and $X_k$. In one atomic move, we perform a series of edge reversals to form a cycle containing $X_j\to X_k$ along the sequence, delete the edge connecting $X_i$ to $X_j$, and orient the edge $X_i\to X_k$. If the likelihood is unchanged, the edge deletion improves the score.\\

\noindent
{\bf Case 2.} This case is shown in the second row of Figure \ref{fig:vedge}. This case involves the case that the cycle over $X_j$ and $X_k$ in the ground truth is a 2-cycle.
If a greedy search algorithm finds one edges between $X_k$ and $X_j$, it can significantly increase the likelihood by adding edges at the location of the virtual edges between $X_i$ and $X_j$ and between $X_l$ and $X_k$. The algorithm would therefore be trapped in a local optimum shown in Figure \ref{fig:vedge}(b) with one more edge than the ground truth shown in Figure \ref{fig:vedge}(c).
To resolve this issue, we propose adding the following search operator: 
Suppose we have triangles over three variables $X_i$, $X_j$ and $X_k$ and $X_l$, $X_j$ and $X_k$, as shown in the figure. In one atomic move, we  delete the edge connecting $X_i$ to $X_j$ and the edge connecting $X_l$ to $X_k$, and add the edge $X_k\to X_j$. If the likelihood is unchanged, the edge deletion improves the score. \\

\begin{figure}[t!]
    \centering
    \subfigure[]{
        \label{fig:virtsupp_2cycle_nominal}
        \includegraphics[scale=0.55]{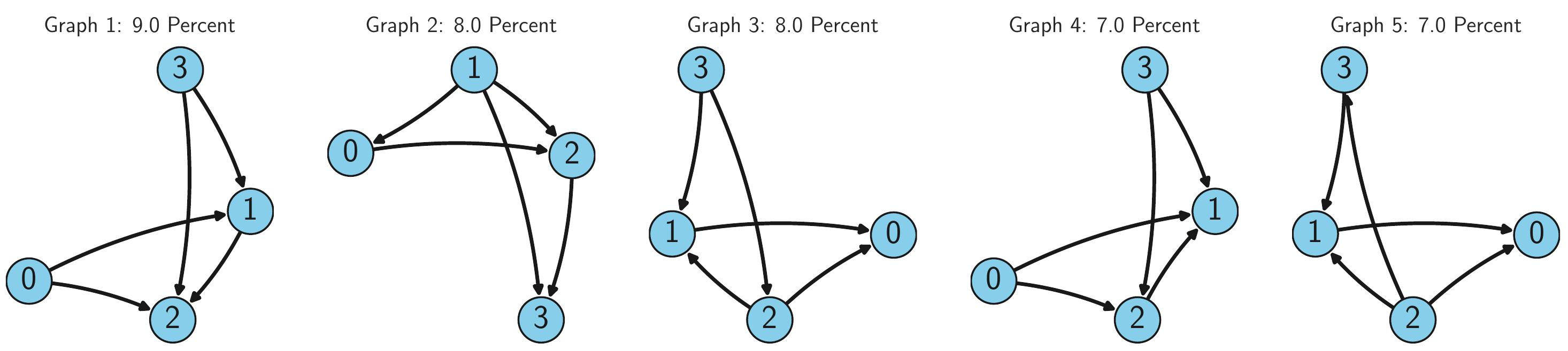}
    }
    \newline
    \subfigure[]{
        \label{fig:virtsupp_2cycle_virtual}
        \includegraphics[scale=0.55]{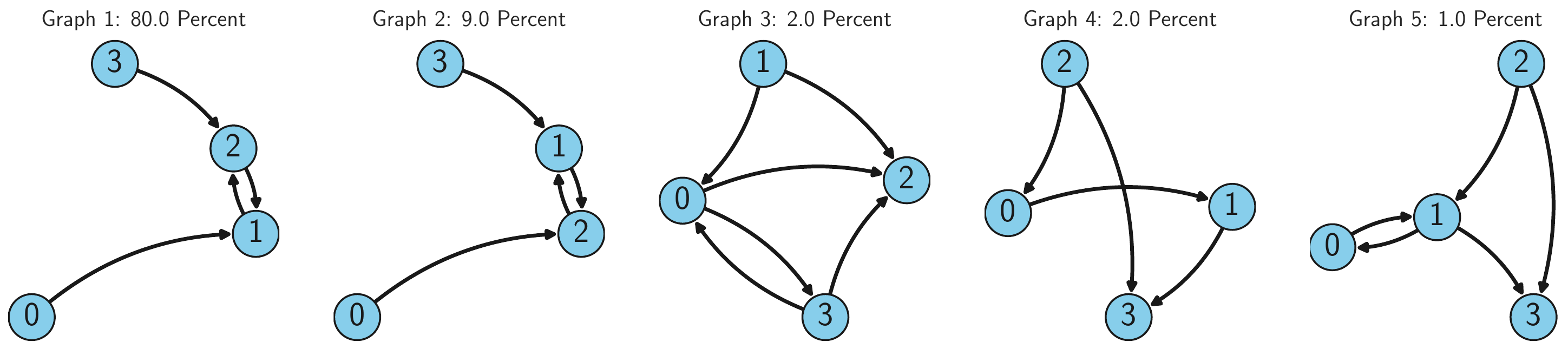}
    }
\caption{Example 1. Comparison of 5 most commonly learned structures.}
\label{fig:vedgesim_2cycle}
\end{figure}

\begin{figure}[t!]
    \centering
    \subfigure[]{
        \centering
        \includegraphics[scale=0.55]{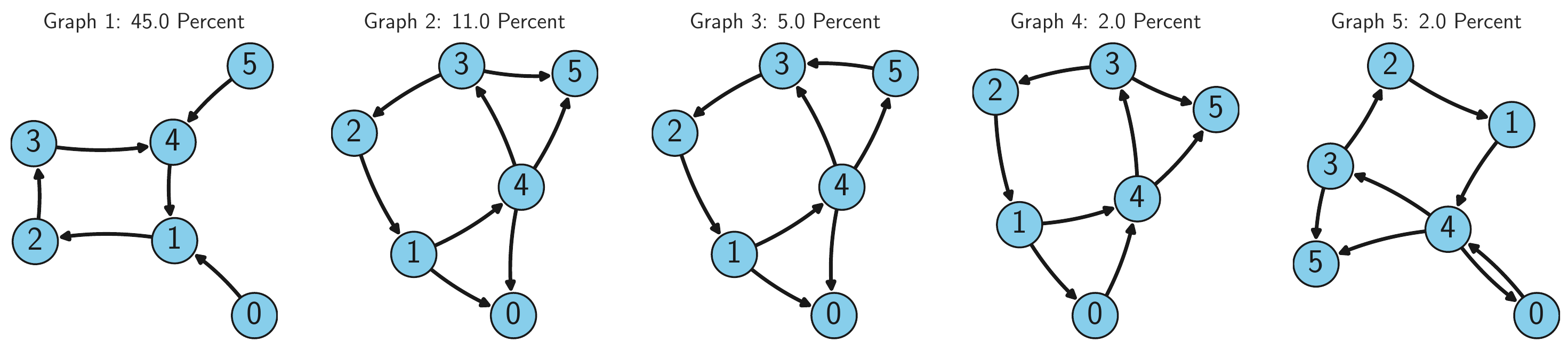}
        \label{fig:virtsupp_4cycle_nominal}
    }
    \newline
    \subfigure[]{
        \centering
        \includegraphics[scale=0.55]{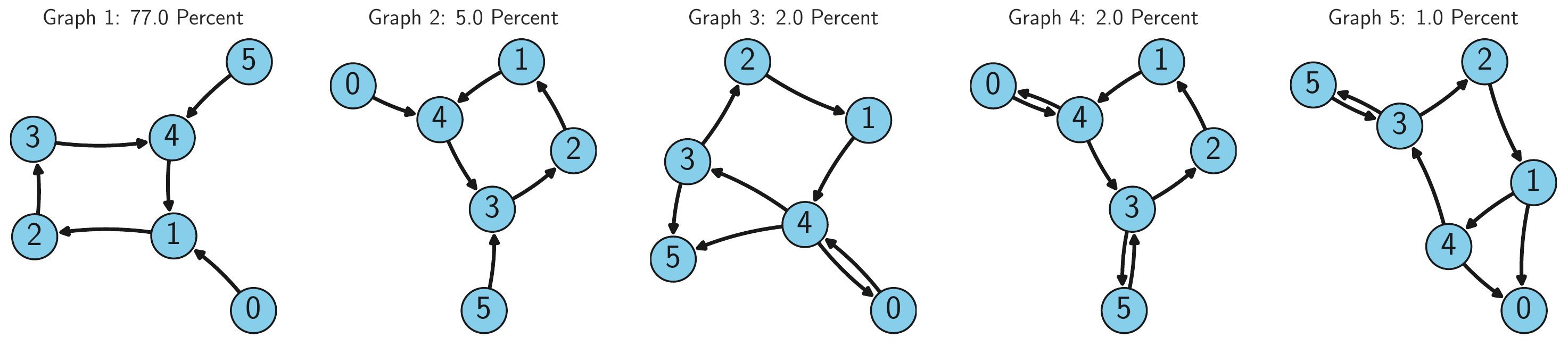}
        \label{fig:virtsupp_4cycle_virtual}
    }
\caption{Example 2. Comparison of 5 most commonly learned structures.}
\label{fig:vedgesim_4cycle}
\end{figure}

In order to evaluate the proposed search operator, we performed two experiments. The first involves the ground truth structure shown in Figure \ref{fig:virtsupp_2cycle_virtual}, Graph 1. This graph has one equivalent structure, which is Graph 2 in the same figure. We run the tabu search algorithm with and without the proposed search operator for 100 instantiations of the edge weights and variances. The 5 most commonly found structures found by tabu search without and with the proposed operator are shown in Figures \ref{fig:virtsupp_2cycle_nominal} and \ref{fig:virtsupp_2cycle_virtual}, respectively. While the proposed algorithm finds an equivalent structure $89\%$ of the time, the nominal tabu search never finds an equivalent structure. 

Next, we consider the ground truth structure shown in Figure \ref{fig:virtsupp_4cycle_virtual}, Graph 1. This structure has one equivalent, which is Graph 2 in the same figure. While the nominal tabu search algorithm finds an equivalent structure $45\%$ of the time, the proposed algorithm is much more reliable, finding an equivalent structure $83\%$ of the time.

\section{Score Decomposability}

When the DG is acyclic, the distribution generated by a linear Gaussian structural equation model satisfies the local Markov property. This implies that the joint distribution can be factorized into the product of the distributions of the variables conditioned on their parents as follows.
\[
P(V)=\prod_{X_i\in V}P(X_i|\Pa(X_i)).
\]
The benefit of this factorization is that the computational complexity of evaluating the effect of operators can be dramatically reduced since a local change in the structure does not change the score of other parts of the DAG. 

In contrast, for the case of cyclic DGs the distribution does not necessarily satisfy the local Markov property.
However, the distribution still satisfies the global Markov property \citep{spirtes1995directed}. Therefore, our search procedure factorizes the joint distribution into the product of conditional distributions. Each of these distributions is over the variables in a maximal strongly connected subgraph (MSCS), conditioned on their parents outside of the MSCS. This can be shown as follows, where an MSCS is denoted by $S$.
\[
P(V)=\prod_{S_i\subseteq V}P(S_i|\Pa(S_i)).
\]
After applying an operation, the likelihoods of all involved MSCSs are updated. Note that an operation can merge several MSCSs or break one into several smaller MSCSs. We perform the updates as follows:
\begin{itemize}
    \item If the change adds an edge from MSCS $S_1$ to $S_2$, These two MSCSs and any MSCS on any path from $S_2$ to $S_1$ will fused into a new large MSCS.
    \item If the change is performed inside an MSCS, the score of the rest of MSCSs do not change.
    \item If the change removes or reverses an edge inside an MSCS, we find the MSCSs in that subset again, as it may be divided into smaller MSCSs.
\end{itemize}

\section{Effect of Sample Size on the Performance}
%Histogram of DG's edge count

In this section, we compare the performance of the discussed structure learning algorithms in the case of $p=5$ variables and three different sample sizes: $n=10^3, 10^4,$ and $10^5$. The results of the comparison are shown in Figure \ref{fig:exps_sample_size}. As can be seen in the figure, the performance of the $\ell_0$-regularized local search methods show marked improvement as sample size is increased.

For all experiments, including those in the main text, we use the following hyperparameters for the search algorithms. For the $\ell_1$-regularized MLE, we use a regularization coefficient of $0.1$, and threshold the learned $B$ matrix at $0.05$. See \citep{koller2009probabilistic} for details on greedy hill search and  tabu search and its parameters. For tabu search, we use a tabu length of 5 for the $p=5$ case and 10 for the $p=20$ and $p=50$ cases. In all cases, we used a tabu search patience of 5.

\begin{figure}[t]
\centering
\begin{minipage}{.25\textwidth}
  \centering
  \includegraphics[width=.99\linewidth,height=33mm]{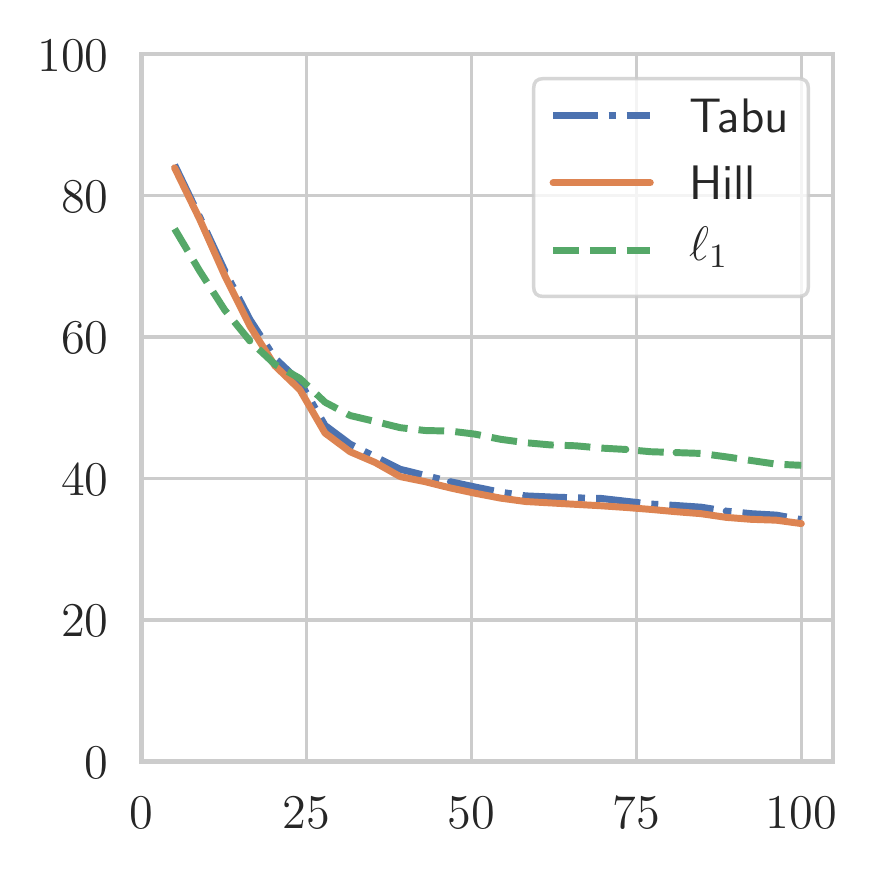}
  %\captionof{figure}{A figure}
  %\label{fig:test1}
\end{minipage}%
\begin{minipage}{.25\textwidth}
  \centering
  \includegraphics[width=.99\linewidth,height=33mm]{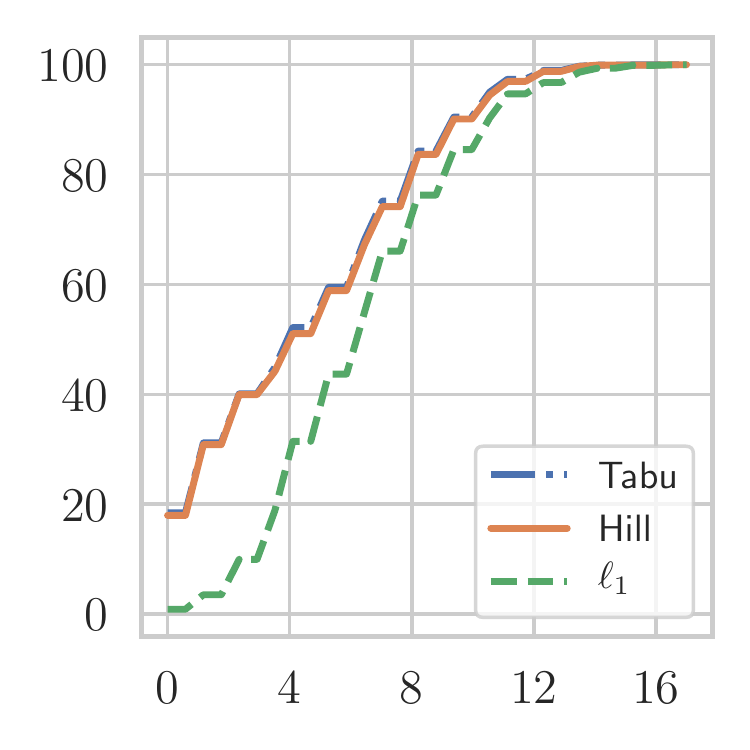}
  %\captionof{figure}{Another figure}
  %\label{fig:test2}
\end{minipage}\vspace{-0mm}
\begin{minipage}{.25\textwidth}
  \centering
  \includegraphics[width=.99\linewidth,height=33mm]{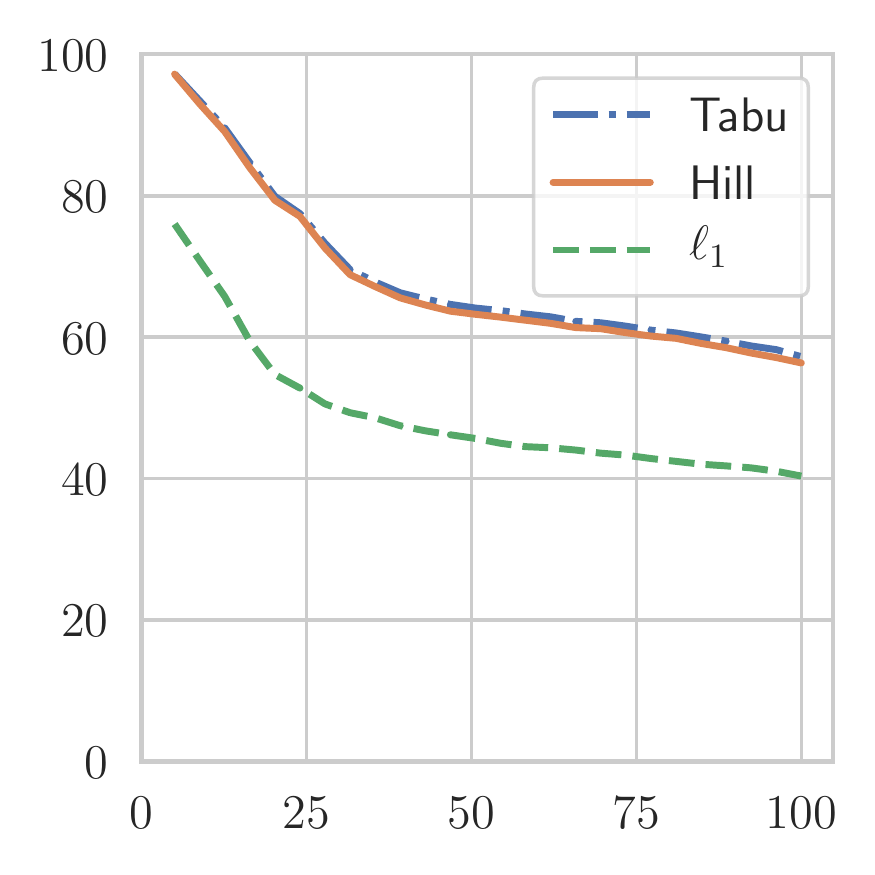}
  %\captionof{figure}{A figure}
  %\label{fig:test1}
\end{minipage}%
\begin{minipage}{.25\textwidth}
  \centering
  \includegraphics[width=.99\linewidth,height=33mm]{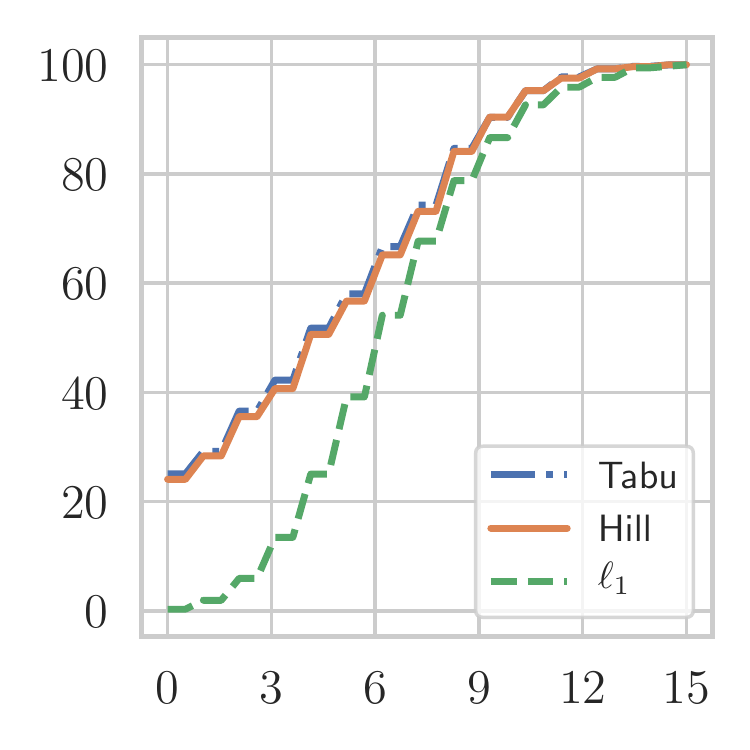}
  %\captionof{figure}{Another figure}
  %\label{fig:test2}
\end{minipage}\vspace{-0mm}
\begin{minipage}{.25\textwidth}
  \centering
  \includegraphics[width=.99\linewidth,height=33mm]{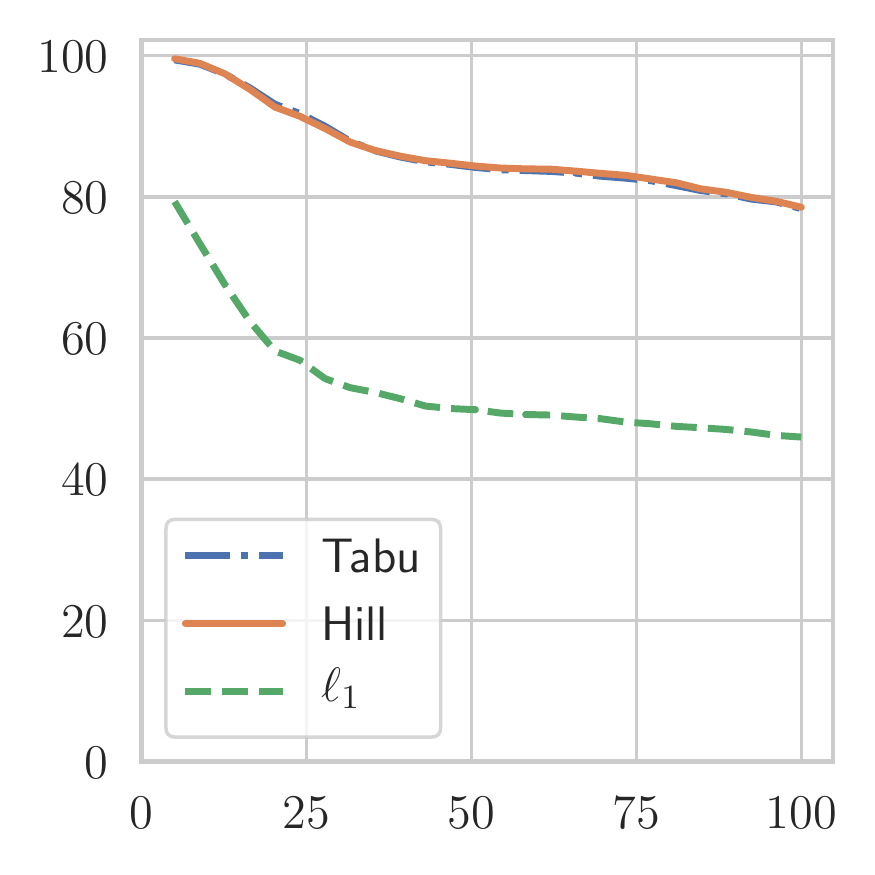}
  %\captionof{figure}{A figure}
  %\label{fig:test1}
\end{minipage}%
\begin{minipage}{.25\textwidth}
  \centering
  \includegraphics[width=.99\linewidth,height=33mm]{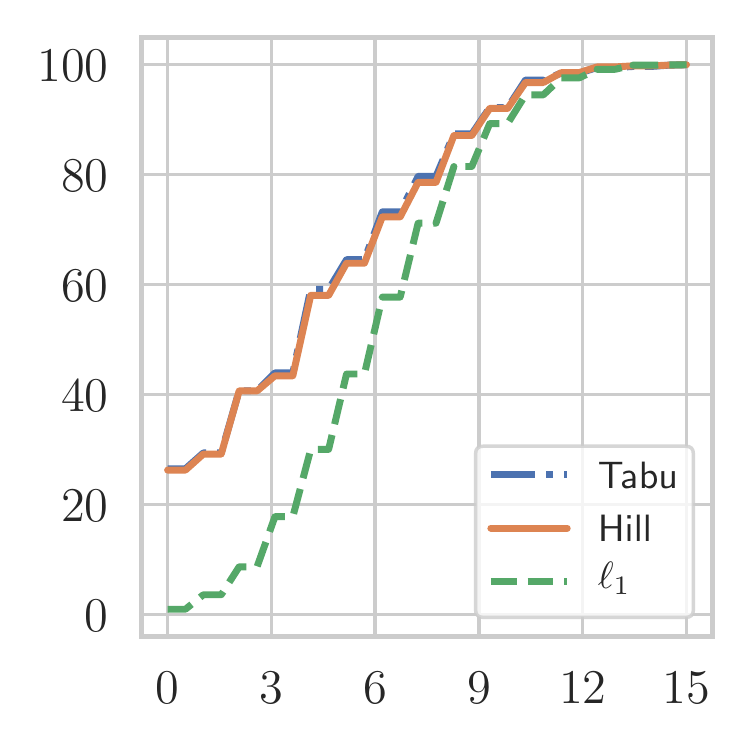}
  %\captionof{figure}{Another figure}
  %\label{fig:test2}
\end{minipage}
%\caption{Ratio of correctly oriented edges versus 
%(a): number of domains for model 1 for IB method,
%(b): number of variables for model 1 for IB method,
%(c): number of domains for model 2 for MC method,
%(d): number of variables for model 2 for MC method.
%}
%\vspace{-5mm}
\caption{Results for $n=10^3,10^4,10^5$, top to bottom. {\bf Left column:} multi-domain evaluation. The percentage of outputs with success rate larger than a certain value is plotted vs. success percentages. {\bf Right column:} SHD evaluation. The percentage of outputs with SHD less than or equal to a certain value is plotted vs. SHD.} 
% \caption{Evaluation results for $p=5,20,50$ from top to bottom. The left column shows the multi-domain evaluation results.
% The x-axis represents the success percentages and the plot shows that what percentage of the outputs had a success rate larger than a certain percentage. 
% For example, for $p=20$, $80\%$ of the outputs were capable of generating more than $25\%$ of the distributions generated by their corresponding ground truth DG.
%Therefore, if one believes that $25\%$ distribution set overlap is required for two DGs to be quasi equivalent, then the results suggest that $80\%$ of the outputs are quasi equivalent to their ground truth.
%The right column shows the SHD evaluation results.The x-axis represents the SHD values and the plot shows that what percentage of the outputs had an SHD less than or equal to a certain value.} 
%\vspace{-4mm}
\label{fig:exps_sample_size}
\end{figure}

\end{document}